\documentclass{article}

\usepackage{microtype}
\usepackage{graphicx}
\usepackage{booktabs} 

\usepackage{hyperref}
\usepackage{mathtools}
\usepackage{appendix}
\usepackage{amsthm}
\usepackage{amsmath}

\usepackage{amsfonts}
\usepackage{multirow}

\usepackage{xcolor,colortbl}
\definecolor{LightCyan}{rgb}{0.88,1,1}
\usepackage[normalem]{ulem} 
\newcommand\hl{\bgroup\markoverwith
  {\textcolor{LightCyan}{\rule[-.5ex]{2pt}{2.5ex}}}\ULon}

\usepackage{algorithm}
\usepackage{algorithmic}

\DeclarePairedDelimiter\abs{\lvert}{\rvert}%
\usepackage{subfig}

\newtheorem{theorem}{Theorem}
\newtheorem{lemma}{Lemma}
\newtheorem{proposition}{Proposition}

\newtheorem{assumption}{Assumption}
\newtheorem{remark}{Remark}

\usepackage[accepted]{icml2020}

\icmltitlerunning{Momentum-Based Policy Gradient Methods}

\begin{document}

\twocolumn[
\icmltitle{Momentum-Based Policy Gradient Methods}




\begin{icmlauthorlist}
\icmlauthor{Feihu Huang}{1}
\icmlauthor{Shangqian Gao}{1}
\icmlauthor{Jian Pei}{2}
\icmlauthor{Heng Huang}{1,3}
\end{icmlauthorlist}

\icmlaffiliation{1}{Department of Electrical and Computer Engineering, University of Pittsburgh, Pittsburgh, USA}
\icmlaffiliation{2}{School of Computing Science, Simon Fraser University, Vancouver, Canada}
\icmlaffiliation{3}{JD Finance America Corporation, Mountain View, CA, USA}

\icmlcorrespondingauthor{Feihu Huang}{huangfeihu2018@gmail.com}
\icmlcorrespondingauthor{Heng Huang}{heng.huang@pitt.edu}

\icmlkeywords{Machine Learning, ICML}

\vskip 0.3in
]



\printAffiliationsAndNotice{}  

\begin{abstract}
In the paper, we propose a class of efficient momentum-based policy gradient methods
for the model-free reinforcement learning,
which use adaptive learning rates and do not require any large batches.
Specifically, we propose a fast important-sampling momentum-based policy gradient (IS-MBPG) method based on a new momentum-based variance reduced technique and the importance sampling technique.
We also propose a fast Hessian-aided momentum-based policy gradient (HA-MBPG) method based on the momentum-based variance reduced technique and the Hessian-aided technique.
Moreover, we prove that both the IS-MBPG and HA-MBPG methods
reach the best known sample complexity of $O(\epsilon^{-3})$ for finding an $\epsilon$-stationary point of the nonconcave performance function, which only require one trajectory at each iteration.
In particular, we present a non-adaptive version of IS-MBPG method, i.e., IS-MBPG*,
which also reaches the best known sample complexity of  $O(\epsilon^{-3})$ without any large batches.
In the experiments, we apply four benchmark tasks to demonstrate the effectiveness of our algorithms.
\end{abstract}

\section{Introduction}
Reinforcement Learning (RL) has achieved great success in solving many sequential decision-making problems
such as autonomous driving \cite{shalev2016safe}, robot manipulation \cite{deisenroth2013survey}, the game of Go \cite{silver2017mastering}
and natural language processing \cite{wang2018deep}.
In general, RL involves a Markov decision process (MDP), where an agent takes actions dictated by a policy
in a stochastic environment over a sequence of time steps, and then maximizes the long-term cumulative rewards to obtain an optimal policy.
Due to easy implementation and avoiding policy degradation, policy gradient method \cite{williams1992simple,sutton2000policy}
is widely used for finding the optimal policy in MDPs,
especially for the high dimensional continuous state and action spaces.
To obtain the optimal policy, policy gradient methods directly maximize the expected total reward
(also called as performance function $J(\theta)$) via using the
stochastic first-order gradient of cumulative rewards.
Recently, policy gradient methods have achieved significant empirical successes
in many challenging deep reinforcement learning applications \cite{li2017deep}
such as playing game of Go and robot manipulation.
\begin{table*}
\vspace*{-6pt}
  \centering
  \caption{ Convergence properties of the representative variance-reduced policy algorithms on the \emph{non-oblivious} model-free RL problem
  for finding an $\epsilon$-stationary point of the nonconcave performance function $J(\theta)$, i.e., $\mathbb{E}\|\nabla J(\theta)\|\leq \epsilon$. Our algorithms (IS-MBPG, IS-MBPG* and HA-MBPG)
  and REINFORCE are \textbf{single-loop} algorithms, while the other algorithms are \textbf{double-loops}, which need the outer-loop and inner-loop mini-batch sizes.
  Note that \citet{papini2018stochastic} only remarked that apply the ADAM algorithm \cite{kingma2014adam}
  to the SVRPG algorithm to obtain an adaptive learning rate, but did not provide any theoretical analysis about this learning rate. In addition, the sample complexity $O(\epsilon^{-4})$ of REINFORCE does not directly come from \cite{williams1992simple}, but follows theoretical results of SGD \cite{ghadimi2013stochastic} (A detailed theoretical analysis is given in the Appendix \ref{Appendix:A4}). }
  \label{tab:1}
  \begin{tabular}{c|c|c|c|c}
  \toprule
 \textbf{Algorithm} & \textbf{Reference} &  \textbf{Sample Complexity}  & \textbf{Batch Size} & \textbf{Adaptive Learning Rate}\\ \hline
  REINFORCE & \citet{williams1992simple} & $O(\epsilon^{-4})$  & $O(\epsilon^{-2})$ &   \\ \hline
  SVRPG & \citet{papini2018stochastic} & $O(\epsilon^{-4})$ & $O(\epsilon^{-2})$ \& $O(\epsilon^{-2})$ &  \\ \hline
  SVRPG & \citet{xu2019improved} & $O(\epsilon^{-10/3})$  & $O(\epsilon^{-4/3})$ \& $O(\epsilon^{-2})$ & \\ \hline
  HAPG & \citet{shen2019hessian} & $O(\epsilon^{-3})$ &  $O(\epsilon^{-1})$ \& $O(\epsilon^{-2})$ &  \\ \hline
  SRVR-PG & \citet{xu2019sample} & $O(\epsilon^{-3})$  & $O(\epsilon^{-1})$ \& $O(\epsilon^{-2})$ &    \\ \hline
  IS-MBPG &Ours &  $O(\epsilon^{-3})$  & $O(1)$ &\checkmark \\ \hline
  HA-MBPG & Ours & $O(\epsilon^{-3})$ & $O(1)$ & \checkmark  \\ \hline
  IS-MBPG* & Ours & $O(\epsilon^{-3})$ & $O(1)$ &  \\ \hline
  \end{tabular}
\end{table*}

Thus, policy gradient methods have regained much interest in reinforcement learning,
and some corresponding algorithms and theory of policy gradient \cite{fellows2018fourier,fujimoto2018addressing,papini2018stochastic,haarnoja2018soft,xu2019improved,shen2019hessian,cheng2019predictor,
cheng2019trajectory,wang2019neural} have been proposed and studied.
Since the classic policy gradient methods (e.g., REINFORCE \cite{williams1992simple}, PGT \cite{sutton2000policy}, GPOMDP \cite{baxter2001infinite} and TRPO \cite{schulman2015trust})
approximate the gradient of the expected total reward based on a batch of sampled trajectories,
they generally suffer from large variance in the estimated gradients, which results in a poor convergence.
Following the standard stochastic gradient methods \cite{robbins1951stochastic,ghadimi2013stochastic},
these gradient-based policy methods require $O(\epsilon^{-4})$ samples for finding an $\epsilon$-stationary point of non-concave performance function $J(\theta)$ (
i.e., $\mathbb{E}\|\nabla J(\theta)\|\leq \epsilon$).
Thus, recently many works have begun to study to reduce variance in the policy gradient methods.
For example, the early variance reduced policy methods \cite{greensmith2004variance,peters2008reinforcement} mainly focused on
using unbiased baseline functions to reduce the variance. \citet{schulman2015high} presented the generalized advantage estimation (GAE)
to discover the balance between bias and variance of policy gradient.
Then \citet{gu2016q} applied both the GAE and linear baseline function to reduce variance.
Recently, \citet{mao2018variance,wu2018variance} proposed the input-dependent
and action-dependent baselines to reduce the variance, respectively.
More recently, \citet{cheng2019predictor} leveraged the predictive models to reduce the variance to accelerate policy learning.

Recently, the variance reduced gradient estimators such as SVRG \cite{johnson2013accelerating,allen2016variance,reddi2016stochastic}, SAGA \cite{defazio2014saga},
SARAH \cite{nguyen2017sarah}, SPIDER \cite{fang2018spider}, SpiderBoost \cite{wang2019spiderboost} and SNVRG \cite{zhou2018stochastic} have been successful in the oblivious supervised learning.
However, the RL optimization problems are \emph{non-oblivious}, i.e., the distribution of the samples is non-stationarity and changes
over time. Thus, \citet{du2017stochastic,xu2017stochastic,wai2019variance} first transform the original non-oblivious policy evaluation problem into some oblivious subproblems,
and then use the existing variance reduced gradient estimators (such as SVRG and SAGA) to solve these subproblems
to reach the goal of reducing the large variance in the original RL problem.
For example, \citet{du2017stochastic} first transforms the empirical policy evaluation problem into a quadratic
convex-concave saddle-point problem via linear function approximation, and then
applies the variants of SVRG and SAGA \cite{palaniappan2016stochastic} to solve this oblivious saddle-point problem.

More recently, \citet{papini2018stochastic,xu2019improved,xu2019sample,shen2019hessian} further have developed some variance reduced policy gradient estimators directly used in the non-oblivious
model-free RL, based on the existing variance reduced techniques such as SVRG and SPIDER used in the oblivious supervised learning.
Moreover, \citet{xu2019improved,xu2019sample,shen2019hessian} have effectively improved the sample complexity by using these variance reduced policy gradients.
For example, two efficient variance reduced policy gradient methods, i.e, SRVR-PG \cite{xu2019sample} and HAPG \cite{shen2019hessian}
have been proposed based on the SARAH/SPIDER, and reach a sharp sample complexity of $O(\epsilon^{-3})$ for finding an
$\epsilon$-stationary point, which improves the vanilla complexity of $O(\epsilon^{-4})$ \cite{williams1992simple}
by a factor of $O(\epsilon^{-1})$.
Since a lower bound of complexity of $O(\epsilon^{-3})$ for recently proposed variance reduction techniques is established in \cite{arjevani2019lower},
both the SRVR-PG and HAPG obtain a near-optimal sample complexity of $O(\epsilon^{-3})$.
However, the practical performances of these variance reduced policy gradient methods are not consistent with their near-optimal sample complexity,
because these methods require large batches and
strict learning rates to achieve this optimal complexity.

In the paper, thus, we propose a class of efficient momentum-based policy gradient methods,
which use adaptive learning rates and do not require any large batches.
Specifically, our algorithms only need one trajectory at each iteration, and use adaptive learning rates based on the current
and historical stochastic gradients.
Note that \citet{pirotta2013adaptive} has studied the adaptive learning rates for policy gradient methods, which only focuses on Gaussian policy.
Moreover, \citet{pirotta2013adaptive} did not consider sample complexity and can not improve it.
While our algorithms not only provide the adaptive learning rates that are suitable for any policies, but also improve sample complexity.
 \vspace*{-5pt}
\subsection*{Contributions}
 \vspace*{-5pt}
Our main contributions are summarized as follows:
 \vspace*{-5pt}
\begin{itemize}
\setlength\itemsep{0em}
\vspace*{-4pt}
\item[1)] We propose a fast important-sampling momentum-based policy gradient (IS-MBPG) method with adaptive learning rate, which builds on a new momentum-based variance reduction technique of STORM/Hybrid-SGD  \citep{cutkosky2019momentum,tran2019hybrid} and the importance sampling technique.
\item[2)] We propose a fast Hessian-aided momentum-based policy gradient (HA-MBPG) method with adaptive learning rate, based on the momentum-based variance reduction technique and the Hessian-aided technique.
\item[3)] We study the sample complexity of our methods, and prove that both the IS-MBPG and HA-MBPG methods reach the best known sample complexity of $O(\epsilon^{-3})$ without any large batches (see Table \ref{tab:1}).
\item[4)] We propose a non-adaptive version of IS-MBPG method, i.e., IS-MBPG*, which has a simple monotonically decreasing learning rate.
 We prove that it also reaches the best known sample complexity of  $O(\epsilon^{-3})$ without any large batches.
\vspace*{-4pt}
\end{itemize}
After our paper is accepted, we find that three related papers \cite{xiong2020non,pham2020hybrid,yuan2020stochastic} more recently
are released on arXiv.
\citet{xiong2020non} has studied the adaptive Adam-type policy gradient (PG-AMSGrad) method, which still suffers from a high sample complexity of $O(\epsilon^{-4})$.
Subsequently, \citet{pham2020hybrid,yuan2020stochastic} have proposed the policy gradient methods, i.e., ProxHSPGA and STORM-PG, respectively,
which also build on the momentum-based variance reduced technique of STORM/Hybrid-SGD.
Although both the ProxHSPGA and STORM-PG reach the best known sample complexity of $O(\epsilon^{-3})$,
these methods still rely on large batch sizes to obtain this sample complexity and do not
provide an effective adaptive learning rate as our methods.
 \vspace*{-3pt}
\subsection*{Notations}
 \vspace*{-3pt}
Let $\|\cdot\|$ denote the vector $\ell_2$ norm and the matrix spectral norm, respectively.
We denote $a_n=O(b_n)$ if $a_n\leq cb_n$ for some constant $c>0$. $\mathbb{E}[X]$ and $\mathbb{V}[X]$
denote the expectation and variance of a random variable $X$, respectively.
$\mathbb{E}_{\tau_t} [\cdot]= \mathbb{E}_{\tau_t} [\cdot|\tau_1,\cdots,\tau_{t-1}]$ for any $t\geq 2$.
 \vspace*{-6pt}
\section{Background}
 \vspace*{-3pt}
In the section, we will review some preliminaries of standard reinforcement learning and policy gradient.
 \vspace*{-6pt}
\subsection{Reinforcement Learning}
 \vspace*{-3pt}
Reinforcement learning is generally modeled as a discrete time Markov Decision Process (MDP): $\mathcal{M}=\{\mathcal{S},\mathcal{A},\mathcal{P},\mathcal{R},\gamma,\rho_0\}$.
Here $\mathcal{S}$ is the state space, $\mathcal{A}$ is the action space, and $\rho_0$ denotes the initial state distribution.
$\mathcal{P}(s'|s,a)$ denotes the probability that the agent transits from the state $s$ to $s'$
under taking the action $a\in \mathcal{A}$. $\mathcal{R}(s,a): \mathcal{S} \times \mathcal{A}\mapsto [-R,R] \ (R>0)$ is the bounded reward function,
\emph{i.e.,} the agent obtain the reward $\mathcal{R}(s,a)$ after it takes the action $a$ at the state $s$, and $\gamma \in (0,1)$ is the discount factor.
The policy $\pi (a|s)$ at the state $s$ is represented by a conditional probability distribution $\pi_{\theta}(a|s)$
associated to the parameter $\theta \in \mathbb{R}^{d}$.

Given a time horizon $H$, the agent can collect a trajectory $\tau=\{s_0,a_0, \cdots, s_{H-1}, a_{H-1}\}$ under any stationary policy.
Following the trajectory $\tau$, a cumulative discounted reward can be given as follows:
\begin{align}
 \mathcal{R}(\tau) = \sum^{H-1}_{h=0}\gamma^h\mathcal{R}(s_h,a_h),
\end{align}
where $\gamma$ is the discount factor. Assume that the policy $\pi_{\theta}$ is parameterized by an unknown parameter $\theta\in \mathbb{R}^d$.
Given the initial distribution $\rho_0=\rho(s_0)$,
the probability distribution over trajectory $\tau$ can be obtain
\begin{align} \label{eq:0}
 p(\tau|\theta) = \rho(s_0)\prod_{h=0}^{H-1}\mathcal{P}(s_{h+1}|s_h,a_h)\pi_{\theta}(a_h|s_h).
\end{align}
 \vspace*{-6pt}
\subsection{Policy Gradient}
The goal of RL is to find an optimal policy $\pi_{\theta}$ that is equivalent to maximize
the expected discounted trajectory reward:
\begin{align} \label{eq:1}
\max_{\theta \in \mathbb{R}^d} J(\theta):= \mathbb{E}_{\tau\sim p(\tau|\theta)} [\mathcal{R}(\tau)]=\int\mathcal{R}(\tau) p(\tau|\theta)d\tau.
\end{align}
Since the underlying distribution $p$ depends on the variable $\theta$ and varies through the whole optimization procedure,
the problem \eqref{eq:1} is a \emph{non-oblivious} learning problem, which is unlike the traditional
supervised learning problems that the underlying distribution $p$ is stationary.
To deal with this problem, the policy gradient method \cite{williams1992simple,sutton2000policy} is a good choice.
Specifically, we first compute the gradient of $J(\theta)$ with respect to $\theta$, and obtain
\begin{align} \label{eq:2}
\nabla J(\theta) & \!=\! \int\! \mathcal{R}(\tau) \nabla p(\tau|\theta)d\tau \!=\!\int\! \mathcal{R}(\tau) \frac{\nabla p(\tau|\theta)}{p(\tau|\theta)}p(\tau|\theta)d\tau \nonumber \\
& \! = \mathbb{E}_{\tau\sim p(\tau|\theta)} \big[\nabla\log p(\tau|\theta)\mathcal{R}(\tau)\big].
\end{align}
Since the distribution $p(\tau | \theta)$ is unknown, we can not compute the exact full gradient of \eqref{eq:2}.
Similar for stochastic gradient descent (SGD), the policy gradient method samples a batch of trajectories $\mathcal{B}=\{\tau_i\}_{i=1}^{|\mathcal{B}|}$ from the distribution $p(\tau | \theta)$
to obtain the stochastic gradient as follows:
\begin{align}
 \hat{\nabla} J(\theta) = \frac{1}{|\mathcal{B}|} \sum_{i\in \mathcal{B}} \nabla\log p(\tau_i|\theta)\mathcal{R}(\tau_i). \nonumber
\end{align}
At the $t$-th iteration, the parameter $\theta$ can be updated:
\begin{align}
 \theta_{t+1} = \theta_t + \eta_t\hat{\nabla}_{\theta} J(\theta),
\end{align}
where $\eta_t>0$ is a learning rate. In addition, since the term $\nabla\log p(\tau_i|\theta)$ is independent of the transition probability $\mathcal{P}$,
we rewrite the stochastic gradient $ \hat{\nabla} J(\theta)$ as follows:
\begin{align} \label{eq:3}
  &\hat{\nabla} J(\theta) = \frac{1}{|\mathcal{B}|} \sum_{i\in \mathcal{B}} g(\tau_i,\theta) \\
  &= \frac{1}{|\mathcal{B}|} \sum_{i\in \mathcal{B}}\big(\sum_{h=0}^{H-1} \nabla_{\theta}\log \pi_{\theta}(a^i_h,s^i_h)\big) \big(\sum_{h=0}^{H-1}\gamma^h\mathcal{R}(s^i_h,a^i_h) \big), \nonumber
\end{align}
where $g(\tau_i,\theta)$ is an unbiased stochastic gradient based on the trajectory $\tau_i$, \emph{i.e.,} $\mathbb{E}[g(\tau_i,\theta)] = \nabla J(\theta)$.
Based on the above gradient estimator in \eqref{eq:3}, we can obtain the existing well-known gradient estimators of
policy gradient such as the REINFORCE, the PGT and the GPOMDP.
Due to $\mathbb{E}[\nabla_{\theta} \log \pi_{\theta}(a,s)]=0$, the REINFORCE adds a constant baseline $b$ and
obtains a gradient estimator as follows:
\begin{align}
 g(\tau_i,\theta) \!=\! \big(\sum_{h=0}^{H-1} \nabla_{\theta}\log \pi_{\theta}(a^i_h,s^i_h)\big) \! \big(\sum_{h=0}^{H-1}\gamma^h\mathcal{R}(s^i_h,a^i_h)\!-\!b \big). \nonumber
\end{align}
Further, considering the fact that the current actions do not rely on the previous rewards,
the PGT refines the REINFORCE and obtains the following gradient estimator:
\begin{align}
 g(\tau_i,\theta) \!=\! \sum_{h=0}^{H-1} \!\sum_{j=h}^{H-1}\big(\gamma^j\mathcal{R}(s^i_j,a^i_j)\!-\!b_j \big)\nabla_{\theta}\log \pi_{\theta}(a^i_h,s^i_h). \nonumber
\end{align}
Meanwhile, the PGT estimator is equivalent to the popular GPOMDP estimator defined as follows:
\begin{align}
 g(\tau_i,\theta) \!=\! \sum_{h=0}^{H-1}\!\sum_{j=0}^{h}\nabla_{\theta}\log \pi_{\theta}(a^i_j,s^i_j)(\gamma^h\mathcal{R}(s^i_h,a^i_h)\!-\!b_h). \nonumber
\end{align}
 \vspace*{-6pt}
\section{Momentum-Based Policy Gradients}
 \vspace*{-6pt}
In the section, we propose a class of fast momentum-based policy gradient methods based on
a new momentum-based variance reduction method, i.e., STORM \cite{cutkosky2019momentum}.
Although the STORM shows its effectiveness in the \emph{oblivious} learning problems,
it is not well suitable for the \emph{non-oblivious} learning problem \label{eq:1},
where the underlying distribution $p(\cdot)$ depends on the variable $\theta$ and varies through the whole optimization procedure.
To deal with this challenge, we will apply two effective techniques, i.e., \emph{importance sampling} \cite{metelli2018policy,papini2018stochastic}
and \emph{Hessian-aided} \cite{shen2019hessian}, and
propose the corresponded policy gradient methods, respectively.
 \vspace*{-6pt}
\subsection{ Important-Sampling Momentum-Based Policy Gradient }
 \vspace*{-5pt}
In the subsection, we propose a fast important-sampling momentum-based policy gradient (IS-MBPG) method based on the importance sampling technique.
Algorithm \ref{alg:1} describes the algorithmic framework of IS-MBPG method.

\begin{algorithm}[tb]
\caption{ Important-Sampling Momentum-Based Policy Gradient (\textbf{IS-MBPG}) Algorithm}
\label{alg:1}
\begin{algorithmic}[1] 
\STATE {\bfseries Input:}  Total iteration $T$, parameters $\{k,m,c\}$ and initial input $\theta_1$; \\
\FOR{$t = 1, 2, \ldots, T$}
\IF{$t = 1$}
\STATE Sample a trajectory $\tau_1$ from $p(\tau |\theta_1)$, and compute $u_1 =  g(\tau_1|\theta_1)$;\\
\ELSE
\STATE Sample a trajectory $\tau_t$ from $p(\tau |\theta_t)$, and compute $u_{t} = \beta_t g(\tau_t|\theta_t) + (1-\beta_t)\big[u_{t-1} + g(\tau_t | \theta_{t})- w(\tau_t|\theta_{t-1},\theta_t) g(\tau_t|\theta_{t-1})\big]$,
       where the importance sampling weight $w(\tau_t|\theta_{t-1},\theta_t)$ can be computed by using \eqref{eq:5}; \\
\ENDIF
\STATE Compute $G_t = \|g(\tau|\theta_t)\|$;
\STATE Compute $\eta_t = \frac{k}{(m+\sum_{i=1}^tG^2_i)^{1/3}}$;
\STATE Update  $\theta_{t+1} = \theta_t + \eta_t u_t$;
\STATE Update  $\beta_{t+1} = c\eta_t^2$;
\ENDFOR
\STATE {\bfseries Output:}  $\theta_{\zeta}$ chosen uniformly random from $\{\theta_t\}_{t=1}^{T}$.
\end{algorithmic}
\end{algorithm}
Since the problem \eqref{eq:1} is \emph{non-oblivious} or \emph{non-stationarity} that the underlying distribution $p(\tau|\theta)$
depends on the variable $\theta$
and varies through the whole optimization procedure,
we have $\mathbb{E}_{\tau\sim p(\tau|\theta)}[g(\tau | \theta)-g(\tau|\theta')] \neq \nabla J(\theta) - \nabla J(\theta')$.
Given $\tau$ sampled from $p(\tau|\theta)$, we define an importance sampling weight
\begin{align} \label{eq:5}
w(\tau|\theta',\theta)=\frac{p(\tau|\theta')}{p(\tau|\theta)}= \prod_{h=0}^{H-1}\frac{\pi_{\theta'}(a_h|s_h)}{\pi_{\theta}(a_h|s_h)}
\end{align}
to obtain $\mathbb{E}_{\tau\sim p(\tau|\theta)}\big[g(\tau | \theta)-w(\tau|\theta',\theta)g(\tau|\theta')\big] = \nabla J(\theta) - \nabla J(\theta')$.
In Algorithm \ref{alg:1}, we use the following momentum-based variance reduced stochastic gradient
\begin{align}
 u_{t} =& (1-\beta_t)\big[ \underbrace{ u_{t-1} + g(\tau_t | \theta_{t}) - w(\tau_t|\theta_{t-1},\theta_t) g(\tau_t|\theta_{t-1})}_{\mbox{SARAH}} \big]\nonumber \\
 & + \beta_t \underbrace{g(\tau_t|\theta_t)}_{\mbox{SGD}}, \nonumber
\end{align}
where $\beta_t\in (0, 1]$.
When $\beta_t=1$, $u_t$ will reduce to a vanilla stochastic policy gradient used in the REINFORCE. When $\beta_t=0$, it will reduce to the SARAH-based stochastic policy gradient used in the SRVR-PG.

Let $e_t = u_t-\nabla J(\theta_t)$. It is easily verified that
\begin{align}
 &\mathbb{E}[e_{t}] \!=\! \mathbb{E}\big[(1-\beta_t)e_{t-1} \!+\! \beta_t( \underbrace{g(\tau_t|\theta_t)\!-\!\nabla J(\theta_t)}_{=T_1}) +(1-\beta_t) \nonumber \\
 & \cdot \!\big(\! \underbrace{g(\tau_t | \theta_{t})\!-\! w(\tau_t|\theta_{t-1},\theta_t) g(\tau_t|\theta_{t-1})\!-\! \nabla J(\theta_t) \!+\! \nabla J(\theta_{t-1})}_{=T_2} \! \big)\!\big] \nonumber \\
 & = (1-\beta_t)\mathbb{E}[e_{t-1}],
\end{align}
where the last equality holds by $\mathbb{E}_{\tau_t\sim p(\tau|\theta_t)}[T_1]=0$ and $\mathbb{E}_{\tau_t\sim p(\tau|\theta_t)}[T_2]=0$.
By Cauchy-Schwarz inequality, we can obtain
\begin{align}
 \mathbb{E}\|e_t\|^2 \leq &(1-\beta_t)^2\mathbb{E}\|e_{t-1}\|^2 +2\beta_t^2\mathbb{E}\|T_1\|^2 \nonumber \\
 &+ 2(1-\beta_t)^2\mathbb{E}\|T_2\|^2.
\end{align}
Since $O(\|T_2\|^2)=O(\|\theta_{t}-\theta_{t-1}\|^2)=O(\eta_t^2\|u_t\|^2)$, we can choose appropriate $\eta_t$ and $\beta_t$ to
reduce the variance of stochastic gradient $u_t$. From the following theoretical results, our IS-MBPG algorithm can generate
the adaptive and monotonically decreasing learning rate $\eta_t \in (0,\frac{1}{2L}]$, and the monotonically decreasing parameter $\beta_t \in (0,1]$.
 \vspace*{-6pt}
\subsection{ Hessian-Aided Momentum-Based Policy Gradient }
In the subsection, we propose a fast Hessian-aided momentum-based policy gradient (HA-MBPG)
method based on the Hessian-aided technique.
Algorithm \ref{alg:2} describes the algorithmic framework of HA-MBPG method.
\begin{algorithm}[tb]
\caption{ Hessian-Aided Momentum-Based Policy Gradient (\textbf{HA-MBPG}) Algorithm}
\label{alg:2}
\begin{algorithmic}[1] 
\STATE {\bfseries Input:}  Total iteration $T$, parameters $\{k,m,c\}$ and initial input $\theta_1$; \\
\FOR{$t = 1, 2, \ldots, T$}
\IF{$t = 1$}
\STATE Sample a trajectory $\tau_1$ from $p(\tau|\theta_1)$, and compute $u_1 = g(\tau_1|\theta_1)$;\\
\ELSE
\STATE Choose $\alpha$ uniformly at random from $[0,1]$, and compute $\theta_t(\alpha) = \alpha \theta_t + (1-\alpha)\theta_{t-1}$;
\STATE Sample a trajectory $\tau_t$ from $p(\tau|\theta_t(\alpha))$, and compute $u_{t} = \beta_t w(\tau_t|\theta_{t},\theta_t(\alpha))g(\tau_t|\theta_t) + (1-\beta_t)\big(u_{t-1} + \Delta_t\big)$, where $w(\tau|\theta_{t},\theta_t(\alpha))$ and $\Delta_t$
can be computed by using \eqref{eq:5} and \eqref{eq:6}, respectively; \\
\ENDIF
\STATE Compute $G_t = \|g(\tau|\theta_t)\|$;
\STATE Compute $\eta_t = \frac{k}{(m+\sum_{i=1}^tG^2_i)^{1/3}}$;
\STATE Update  $\theta_{t+1} = \theta_t + \eta_t u_t$;
\STATE Update  $\beta_{t+1} = c\eta_t^2$;
\ENDFOR
\STATE {\bfseries Output:}  $\theta_{\zeta}$ chosen uniformly random from $\{\theta_t\}_{t=1}^{T}$.
\end{algorithmic}
\end{algorithm}

In Algorithm \ref{alg:2}, at the $7$-th step, we use an unbiased term $\Delta^t$ \big(\emph{i.e.,} $\mathbb{E}_{\tau_t\sim p(\tau|\theta_t(\alpha))}[\Delta^t]
= \nabla J(\theta_t) - \nabla J(\theta_{t-1})$\big) instead of the biased term $g(\tau | \theta_{t})-g(\tau|\theta_{t-1})$.
To construct the term $\Delta^t$, we first assume that the function $J(\theta)$ is twice differentiable as in \cite{furmston2016approximate,shen2019hessian}.
By the Taylor's expansion (or Newton-Leibniz formula), the gradient difference $\nabla J(\theta_t) - \nabla J(\theta_{t-1})$ can be written as
\begin{align} \label{eq:13}
\nabla J(\theta_t) - \nabla J(\theta_{t-1}) = \big[\int^1_0\nabla^2 J(\theta_t(\alpha))d\alpha\big]v_t,
\end{align}
where $v_t = \theta_t-\theta_{t-1}$ and $\theta_t(\alpha) = \alpha\theta_t + (1-\alpha)\theta_{t-1}$ for some $\alpha\in [0,1]$.
Following \cite{furmston2016approximate,shen2019hessian}, we obtain the policy Hessian $\nabla^2 J(\theta)$ as follows:
\begin{align}
&\nabla^2 J(\theta) = \mathbb{E}_{\tau\sim p(\tau|\theta)}\big[\big(\nabla \log p(\tau|\theta)\nabla
\log p(\tau|\theta)^T\nonumber \\
& \quad +\nabla^2 \log p(\tau|\theta)\big)\mathcal{R}(\tau)\big] \nonumber \\
& = \mathbb{E}_{\tau\sim p(\tau|\theta)}\big[\nabla \Phi(\tau|\theta)\nabla \log p(\tau|\theta)^T +\nabla^2 \Phi(\tau|\theta)\big], \nonumber
\end{align}
where $\Phi(\tau|\theta) = \sum_{h=0}^{H-1}\sum_{j=h}^{H-1}\gamma^jr(s_j,a_j)\log\pi_{\theta}(a_h,s_h)$.
Given the random tuple $(\alpha,\tau)$, where $\alpha$ samples uniformly from $[0,1]$ and $\tau$ samples from the distribution
$p(\tau|\theta_t(\alpha))$, we can construct $\Delta_t$ as follows:
\begin{align} \label{eq:6}
 \Delta_t: = \hat{\nabla}^2(\theta_t(\alpha),\tau)v_t,
\end{align}
where $\mathbb{E}_{\tau\sim p(\tau|\theta_t(\alpha))}[\hat{\nabla}^2(\theta_t(\alpha),\tau)] = \nabla^2J(\theta_t(\alpha))$ and
\begin{align}
 \hat{\nabla}^2(\theta_t,\tau) = & \nabla \Phi(\tau|\theta_t(\alpha))\nabla \log p(\tau|\theta_t(\alpha))^T \nonumber \\
 & + \nabla^2 \Phi(\tau|\theta_t(\alpha)). \nonumber
\end{align}
Note that $\mathbb{E}_{\alpha \sim U[0,1]} [ \nabla^2J(\theta_t(\alpha))]= \int^1_0\nabla^2 J(\theta_t(\alpha))d\alpha$
implies the unbiased estimator $\nabla^2 J(\theta(\bar{\alpha}))$ with $\bar{\alpha}$ uniformly sampled from $[0,1]$. Given $\bar{\alpha}$, we have $\mathbb{E}_{\tau\sim p(\tau|\theta_t(\bar{\alpha}))}[\hat{\nabla}^2(\theta_t(\bar{\alpha}),\tau)] = \nabla^2J(\theta_t(\bar{\alpha}))$.
According to the equation \eqref{eq:13}, thus we have $\mathbb{E}_{\alpha \sim U[0,1], \ \tau\sim p(\tau|\theta_t(\alpha)) }[\Delta_t]=\nabla J(\theta_t) - \nabla J(\theta_{t-1})$, where $U[0,1]$ denotes the uniform distribution over $[0,1]$.

Next, we rewrite \eqref{eq:6} as follows:
\begin{align}  \label{eq:7}
 \Delta_t &= \big(\nabla \log p(\tau|\theta_t(\alpha))^Tv_t\big)\nabla \Phi(\tau|\theta_t(\alpha)) \nonumber \\
 &\quad + \nabla^2 \Phi(\tau|\theta_t(\alpha))v_t.
\end{align}
Considering the second term in \eqref{eq:7} is a time-consuming Hessian-vector product, in practice, we
use can the finite difference method to estimate $ \nabla^2 \Phi(\tau|\theta_t(\alpha))v_t$ as follows:
\begin{align}
 & \nabla^2 \Phi(\tau|\theta_t(\alpha))v_t \nonumber \\
 & \approx \frac{\nabla \Phi(\tau|\theta_t(\alpha) + \delta v_t) - \nabla \Phi(\tau|\theta_t(\alpha) - \delta v_t)}{2\delta}v_t \nonumber \\
 & = \nabla^2 \Phi(\tau|\tilde{\theta}_t(\alpha))v_t,
\end{align}
where $\delta>0$ is very small and $\tilde{\theta}_t(\alpha)\in \big[\theta_t(\alpha)-\delta v_t,\theta_t(\alpha)+\delta v_t]$ is obtained
by the mean-value theorem.
Suppose $\Phi(\tau|\theta)$ is $L_2$-second-order smooth, we can upper bound the approximated error:
\begin{align}
 \|\nabla^2 \Phi(\tau|\theta_t(\alpha))v_t-\nabla^2 \Phi(\tau|\tilde{\theta}_t(\alpha))v_t\| \leq L_2\|v_t\|\delta.
\end{align}
Thus, we take a sufficiency small $\delta$ to obtain arbitrarily small approximated error.

In Algorithm \ref{alg:2}, we use the following momentum-based variance reduced stochastic gradient
\begin{align}
 u_{t} = \beta_t w(\tau|\theta_{t},\theta_t(\alpha))g(\tau|\theta_t) + (1-\beta_t)\big(u_{t-1} + \Delta_t\big), \nonumber
\end{align}
where $\beta_t\in (0, 1]$.
When $\beta_t=1$, $u_t$ will reduce to a vanilla stochastic policy gradient used in the REINFORCE. When $\beta_t=0$, it will reduce to the Hessian-aided stochastic policy gradient used in the HAPG.

Let $e_t = u_t-\nabla J(\theta_t)$. It is also easily verified that
\begin{align}
\mathbb{E}[e_{t}] & = \mathbb{E}\big[(1\!-\!\beta_t)e_{t-1} \!+\! \beta_t( \underbrace{w(\tau|\theta_{t},\theta_t(\alpha))g(\tau|\theta_t)\!-\!J(\theta_t)}_{=T_3}) \nonumber \\
 & \quad +(1-\beta_t)\big( \underbrace{\Delta_t- J(\theta_t) + J(\theta_{t-1})}_{=T_4} \big)\big] \nonumber \\
 & = (1-\beta_t)\mathbb{E}[e_{t-1}],
\end{align}
where the last equality holds by $\mathbb{E}_{\tau\sim p(\tau|\theta_t(\alpha))}[T_3]=0$ and $\mathbb{E}_{\tau\sim p(\tau|\theta_t(\alpha))}[T_4]=0$.
Similarly, by Cauchy-Schwarz inequality, we can obtain
\begin{align}
 \mathbb{E}\|e_t\|^2 \leq &(1-\beta_t)^2\mathbb{E}\|e_{t-1}\|^2 +2\beta_t^2\mathbb{E}\|T_3\|^2 \nonumber \\
 &+ 2(1-\beta_t)^2\mathbb{E}\|T_4\|^2.
\end{align}
Since $O(\|T_4\|^2)=O(\|\theta_{t}-\theta_{t-1}\|^2)=O(\eta_t^2\|u_t\|^2)$, we can choose appropriate $\eta_t$ and $\beta_t$ to
reduce the variance of stochastic gradient $u_t$. From the following theoretical results, our HA-MBPG algorithm can also generate
the adaptive and monotonically decreasing learning rate $\eta_t \in (0,\frac{1}{2L}]$, and the monotonically decreasing parameter $\beta_t \in (0,1]$.
 \vspace*{-6pt}
\subsection{Non-Adaptive IS-MBPG*}
 \vspace*{-5pt}
In this subsection, we propose a non-adaptive version of IS-MBPG algorithm, i.e., IS-MBPG*.
The IS-MBPG* algorithm is given in Algorithm \ref{alg:3}.
Specifically, Algorithm \ref{alg:3} applies
a simple monotonically decreasing learning rate $\eta_t$, which only depends on
the number of iteration $t$.

\begin{algorithm}[tb]
\caption{ IS-MBPG* Algorithm}
\label{alg:3}
\begin{algorithmic}[1] 
\STATE {\bfseries Input:}  Total iteration $T$, parameters $\{k,m,c\}$ and initial input $\theta_1$; \\
\FOR{$t = 1, 2, \ldots, T$}
\IF{$t = 1$}
\STATE Sample a trajectory $\tau_1$ from $p(\tau |\theta_1)$, and compute $u_1 =  g(\tau_1|\theta_1)$;\\
\ELSE
\STATE Sample a trajectory $\tau_t$ from $p(\tau |\theta_t)$, and compute $u_{t} = \beta_t g(\tau_t|\theta_t) + (1-\beta_t)\big[u_{t-1} + g(\tau_t | \theta_{t})- w(\tau_t|\theta_{t-1},\theta_t) g(\tau_t|\theta_{t-1})\big]$; \\
\ENDIF
\STATE Compute $\eta_t = \frac{k}{(m+t)^{1/3}}$;
\STATE Update  $\theta_{t+1} = \theta_t + \eta_t u_t$;
\STATE Update  $\beta_{t+1} = c\eta_t^2$;
\ENDFOR
\STATE {\bfseries Output:}  $\theta_{\zeta}$ chosen uniformly random from $\{\theta_t\}_{t=1}^{T}$.
\end{algorithmic}
\end{algorithm}
\section{Convergence Analysis}
In this section, we will study the convergence properties of our algorithms, i.e., IS-MBPG, HA-MBPG and IS-MBPG*.
\emph{All related proofs are provided in supplementary document.}
We first give some assumptions as follows:
\begin{assumption}
Gradient and Hessian matrix of function $\log\pi_{\theta}(a|s)$ are bounded, \emph{i.e.,} there exist constants $M_g, M_h >0$ such that
\begin{align}
 \|\nabla_{\theta}\log\pi_{\theta}(a|s)\| \leq M_g, \ \|\nabla^2_{\theta}\log\pi_{\theta}(a|s)\| \leq M_h.
\end{align}
\end{assumption}
\begin{assumption}
Variance of stochastic gradient $g(\tau|\theta)$ is bounded, \emph{i.e.,} there exists a constant $\sigma >0$, for all $\pi_{\theta}$
such that $\mathbb{V}(g(\tau|\theta)) = \mathbb{E}\|g(\tau|\theta)-\nabla J(\theta)\|^2 \leq \sigma^2$.
\end{assumption}
\begin{assumption}
Variance of importance sampling weight $w(\tau|\theta_1,\theta_2)=p(\tau|\theta_1)/p(\tau|\theta_2)$ is bounded, \emph{i.e.,} there exists a constant $W >0$,
it follows $\mathbb{V}(w(\tau|\theta_1,\theta_2)) \leq W$ for any $\theta_1, \theta_2 \in \mathbb{R}^d$ and $\tau\sim p(\tau|\theta_2)$.
\end{assumption}
Assumptions 1 and 2 have been commonly used in the convergence analysis of
policy gradient algorithms \cite{papini2018stochastic,xu2019improved,xu2019sample,shen2019hessian}.
Assumption 3 has been used in the study of variance reduced policy gradient algorithms
\cite{papini2018stochastic,xu2019improved,xu2019sample}.
Note that the bounded importance sampling weight in Assumption 3 might be
violated in practice. For example, when using neural networks (NNs) as the policy, small perturbations in $\theta$ might raise a large gap in
the point probability due to some activation functions in NNs, which results in very large importance sampling weights.
Thus, we generally clip the importance sampling weights to make our algorithms more effective.
Based on Assumption 1, we give some useful properties of stochastic gradient $g(\tau|\theta)$
and $\hat{\nabla}^2(\theta_t,\tau)$, respectively.
\begin{proposition} \label{pro:1}
(Proposition 4.2 in \cite{xu2019sample}) Suppose $g(\tau|\theta)$ is the PGT estimator. By Assumption 1, we have
\begin{itemize}
 \item[1)] $g(\tau|\theta)$ is $\hat{L}$-Lipschitz differential, i.e., $\|g(\tau|\theta)-g(\tau|\theta')\|\leq L\|\theta-\theta'\|$ with $\hat{L}=M_hR/(1-\gamma)^2$;
 \item[2)] $J(\theta)$ is $\hat{L}$-smooth, i.e., $\|\nabla^2 J(\theta)\|\leq \hat{L}$;
 \item[3)] $g(\tau|\theta)$ is bounded, i.e., $\|g(\tau|\theta)\|\leq G$ for all $\theta\in \mathbb{R}^d$ with $G=M_gR/(1-\gamma)^2$.
\end{itemize}
\end{proposition}
Since $\|\nabla J(\theta)\|=\|\mathbb{E}[g(\tau|\theta)]\|\leq \mathbb{E}\|g(\tau|\theta)\|\leq G$, Proposition \ref{pro:1} implies that $J(\theta)$ is $G$-Lipschitz.
Without loss of generality, we use the PGT estimator to generate the gradient $g(\tau|\theta_t)$ in our algorithms,
so $G_t = \|g(\tau|\theta_t)\| \leq G$.
\begin{proposition} \label{pro:2}
(Lemma 4.1 in \cite{shen2019hessian}) Under Assumption 1, we have for all $\theta$
\begin{align}
 \|\hat{\nabla}^2(\theta,\tau)\|^2 \leq \frac{H^2M_g^4R^2+M_h^2R^2}{(1-\gamma)^4}=\tilde{L}^2.
\end{align}
\end{proposition}
Since $\|\nabla^2J(\theta)\|=\|\mathbb{E}[\hat{\nabla}^2(\theta,\tau)]\|\leq \mathbb{E}\|\hat{\nabla}^2(\theta,\tau)\|\leq \tilde{L}$,
Proposition \ref{pro:2} implies that $J(\theta)$ is $\tilde{L}$-smooth.
Let $L=\max(\hat{L},\tilde{L})$, so $J(\theta)$ is $L$-smooth.
\subsection{Convergence Analysis of IS-MBPG Algorithm}
In the subsection, we analyze the convergence properties of the IS-MBPG algorithm.
The detailed proof is provided in Appendix  \ref{Appendix:A1}.
For notational simplicity, let $B^2 = L^2 + 2G^2C^2_w$ with $C_w = \sqrt{H(2HM_g^2+M_h)(W+1)}$.
\begin{theorem} \label{th:1}
Assume that the sequence $\{\theta_t\}_{t=1}^T$ be generated from Algorithm \ref{alg:1}. Set $k=O(\frac{G^{2/3}}{L})$, $c=\frac{G^2}{3k^3L}+104B^2$,
$m = \max\{2G^2,(2Lk)^3,(\frac{ck}{2L})^3\}$ and $\eta_0 = \frac{k}{m^{1/3}}$,
 we have
\begin{align}
  \mathbb{E}\|\nabla J(\theta_\zeta)\| \leq \frac{\sqrt{2\Omega}m^{1/6} + 2\Omega^{3/4}}{\sqrt{T}} + \frac{2\sqrt{\Omega}\sigma^{1/3}}{T^{1/3}},
\end{align}
 where $\Omega=\frac{1}{k}\big(16(J^* - J(\theta_1)) + \frac{m^{1/3}}{8B^2k}\sigma^2 + \frac{c^2k^{3}}{4B^2}\ln(T+2)\big)$ with $J^*=\sup_{\theta}J(\theta) <+\infty$.
\end{theorem}
\begin{remark}
 Since $\Omega=O(\ln(T))$,
 Theorem \ref{th:1} shows that the IS-MBPG algorithm has $O(\sqrt{\ln(T)}/T^{\frac{1}{3}})$ convergence rate. The IS-MBPG algorithm needs $1$ trajectory
 to estimate the stochastic policy gradient $u_t$ at each iteration,
 and needs $T$ iterations. Without loss of generality, we omit a relative small term $\sqrt{\ln(T)}$.
 By $T^{-\frac{1}{3}} \leq \epsilon$,
 we choose $T=\epsilon^{-3}$. Thus, the IS-MBPG has the sample complexity of $1\cdot T = O(\epsilon^{-3})$
 for finding an $\epsilon$-stationary point.
\end{remark}
\subsection{Convergence Analysis of HA-MBPG Algorithm}
In the subsection, we study the convergence properties of the HA-MBPG algorithm.
 The detailed proof is provided in Appendix  \ref{Appendix:A2}.
\begin{theorem} \label{th:2}
Assume that the sequence $\{\theta_t\}_{t=1}^T$ be generated from Algorithm \ref{alg:2}, and let $k=O(\frac{G^{2/3}}{L})$, $c=\frac{G^2}{3k^3L}+52L^2$,
$m = \max\{2G^2,(2Lk)^3,(\frac{ck}{2L})^3\}$ and $\eta_0 = \frac{k}{m^{1/3}}$, we have
\begin{align}
  \mathbb{E}\|\nabla J(\theta_\zeta)\| \leq \frac{\sqrt{2\Lambda}m^{1/6} + 2\Lambda^{3/4}}{\sqrt{T}} + \frac{2\sqrt{\Lambda}\sigma^{1/3}}{T^{1/3}}, \nonumber
\end{align}
 where $\Lambda=\frac{1}{k}\big(16(J^* - J(\theta_1)) + \frac{m^{1/3}}{4L^2k}\sigma^2 + \frac{(W+1)c^2k^{3}}{2L^2}\ln(T+2)\big)$ with $J^*=\sup_{\theta}J(\theta) <+\infty$.
\end{theorem}
\begin{remark}
 Since $\Lambda=O(\ln(T))$,
 Theorem \ref{th:2} shows that the HA-MBPG algorithm has $O(\sqrt{\ln(T)}/T^{\frac{1}{3}})$ convergence rate. The HA-MBPG algorithm needs $1$ trajectory
 to estimate the stochastic policy gradient $u_t$ at each iteration,
 and needs $T$ iterations. Without loss of generality, we omit a relative small term $\sqrt{\ln(T)}$. By $T^{-\frac{1}{3}} \leq \epsilon$,
 we choose $T=\epsilon^{-3}$. Thus, the HA-MBPG has the sample complexity of $1\cdot T = O(\epsilon^{-3})$
 for finding an $\epsilon$-stationary point.
\end{remark}
\subsection{Convergence Analysis of IS-MBPG* Algorithm}
In the subsection, we give the convergence properties of the IS-MBPG* algorithm.
 The detailed proof is provided in Appendix  \ref{Appendix:A3}.
\begin{theorem} \label{th:3}
Assume that the sequence $\{\theta_t\}_{t=1}^T$ be generated from Algorithm \ref{alg:3},
and let $B^2 = L^2 + 2G^2C^2_w$, $k> 0$ $c=\frac{1}{3k^3L}+104B^2$,
$m = \max\{2,(2Lk)^3,(\frac{ck}{2L})^3\}$ and $\eta_0 = \frac{k}{m^{1/3}}$, we have
\begin{align}
  \mathbb{E}\|\nabla J(\theta_\zeta)\|=\frac{1}{T}\sum_{t=1}^T\mathbb{E}\|\nabla J(\theta_t)\| \leq \frac{\sqrt{\Gamma}m^{1/6}}{\sqrt{T}} + \frac{\sqrt{\Gamma}}{T^{1/3}}, \nonumber
\end{align}
 where $\Gamma=\frac{1}{k}\big(16(J^* - J(\theta_1)) + \frac{m^{1/3}}{8B^2k}\sigma^2 + \frac{c^2k^{3}\sigma^2}{4B^2}\ln(T+2)\big)$ with $J^*=\sup_{\theta}J(\theta) <+\infty$.
\end{theorem}
\begin{remark}
 Since $\Gamma=O(\ln(T))$,
 Theorem \ref{th:3} shows that the IS-MBPG* algorithm has $O(\sqrt{\ln(T)}/T^{\frac{1}{3}})$ convergence rate.
 The IS-MBPG* algorithm needs $1$ trajectory
 to estimate the stochastic policy gradient $u_t$ at each iteration,
 and needs $T$ iterations. Without loss of generality, we omit a relative small term $\sqrt{\ln(T)}$.
 By $T^{-\frac{1}{3}} \leq \epsilon$,
 we choose $T=\epsilon^{-3}$. Thus, the IS-MBPG* also has the sample complexity of $1\cdot T = O(\epsilon^{-3})$
 for finding an $\epsilon$-stationary point.
\end{remark}
\begin{figure}[tb]
\centering
\subfloat[CartPole]{%
  \includegraphics[clip,width=0.42\columnwidth,height=0.3\columnwidth]{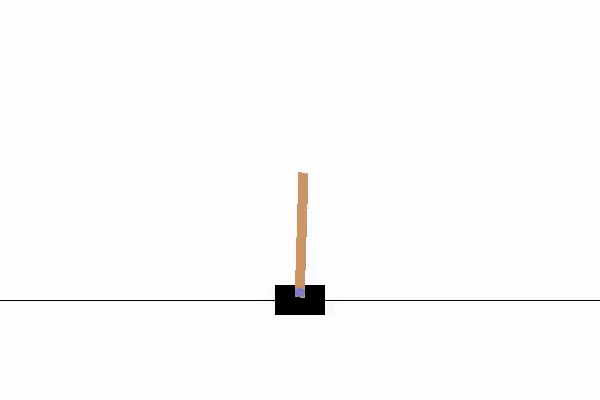}%
}
\hfil
\subfloat[Walker]{%
  \includegraphics[clip,width=0.42\columnwidth,height=0.3\columnwidth]{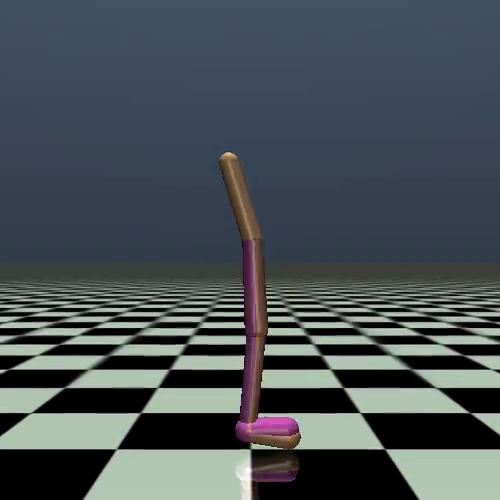}%
}
\\
\subfloat[Hopper]{%
  \includegraphics[clip,width=0.42\columnwidth, height=0.3\columnwidth]{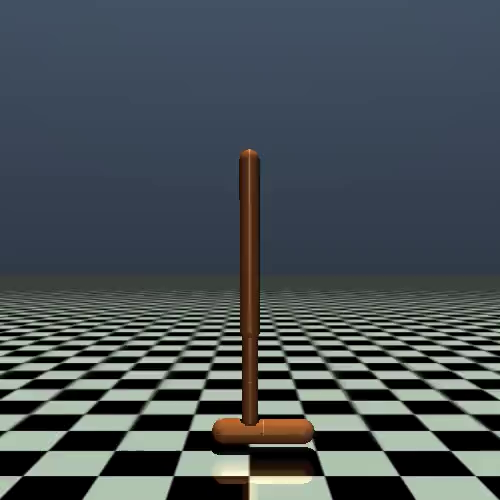}%
}
\hfil
\subfloat[HalfCheetah]{%
  \includegraphics[clip,width=0.42\columnwidth,height=0.3\columnwidth]{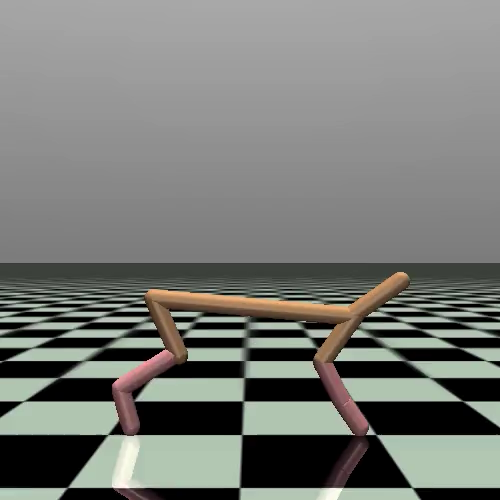}%
}
\caption{Four environments we used. (a) Cartpole: balance a pole on a cart;
(b) Walker: make a 2D robot walk;
(c) Hopper: make a 2D robot hop;
(d) HalfCheetah: make a 2D cheetah robot run.}
\label{fig:1}
 \vspace*{-10pt}
\end{figure}

\begin{figure*}[tb]
\centering
\subfloat[CartPole]{%
  \includegraphics[clip,width=0.38\textwidth]{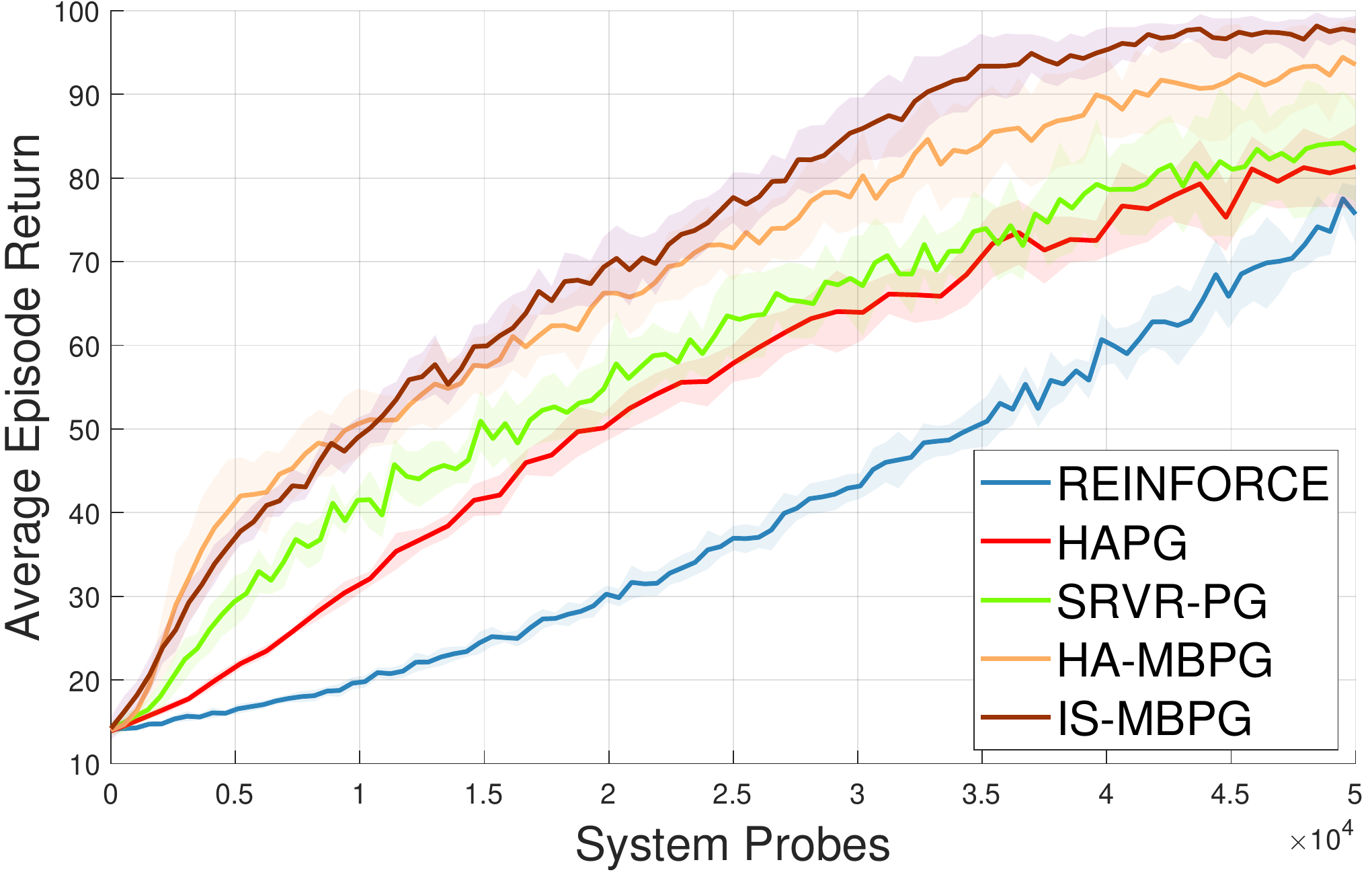}%
}
\hfil
\subfloat[Walker]{%
  \includegraphics[clip,width=0.38\textwidth]{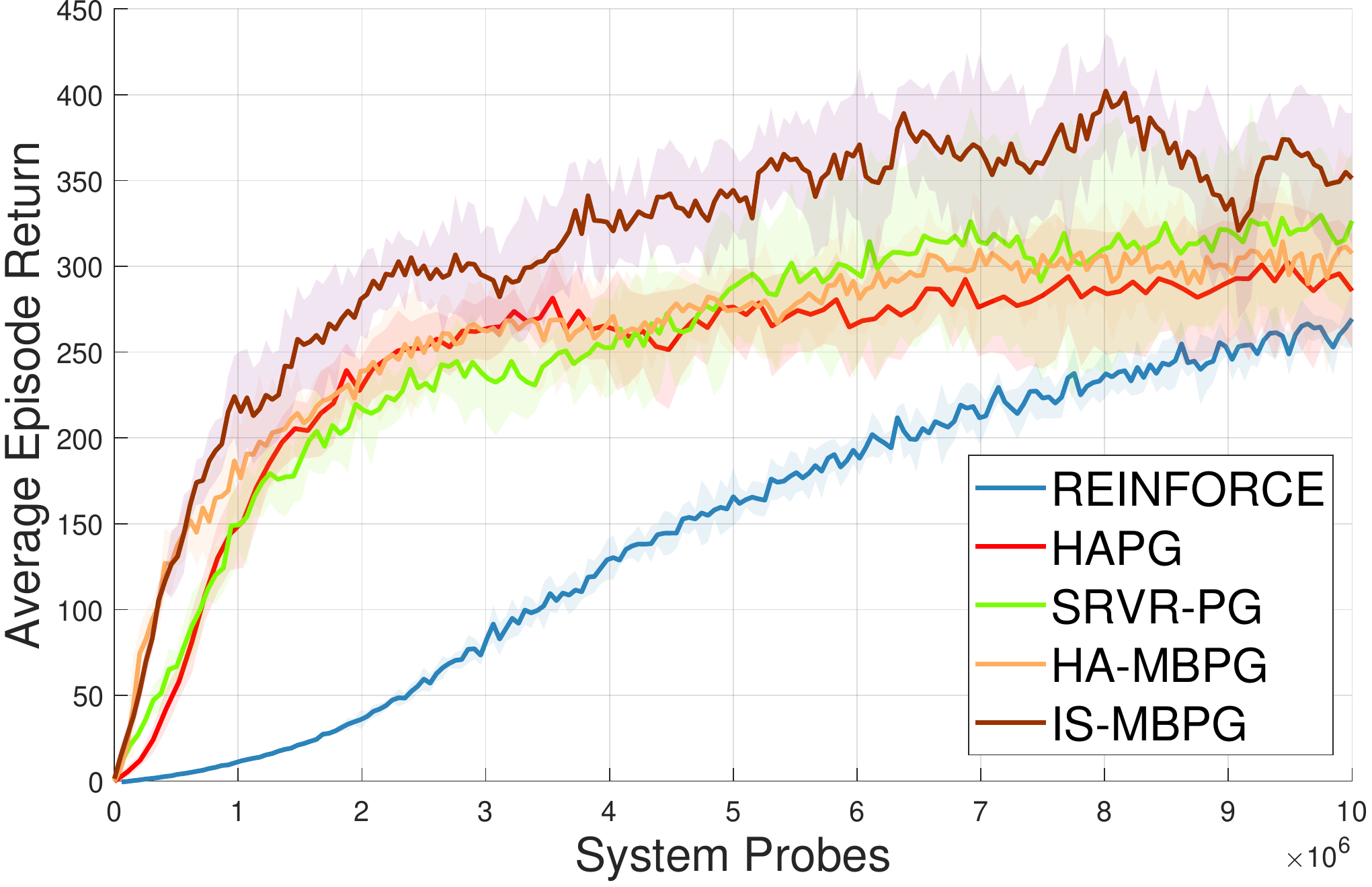}%
}
\\
\subfloat[Hopper]{%
  \includegraphics[clip,width=0.38\textwidth]{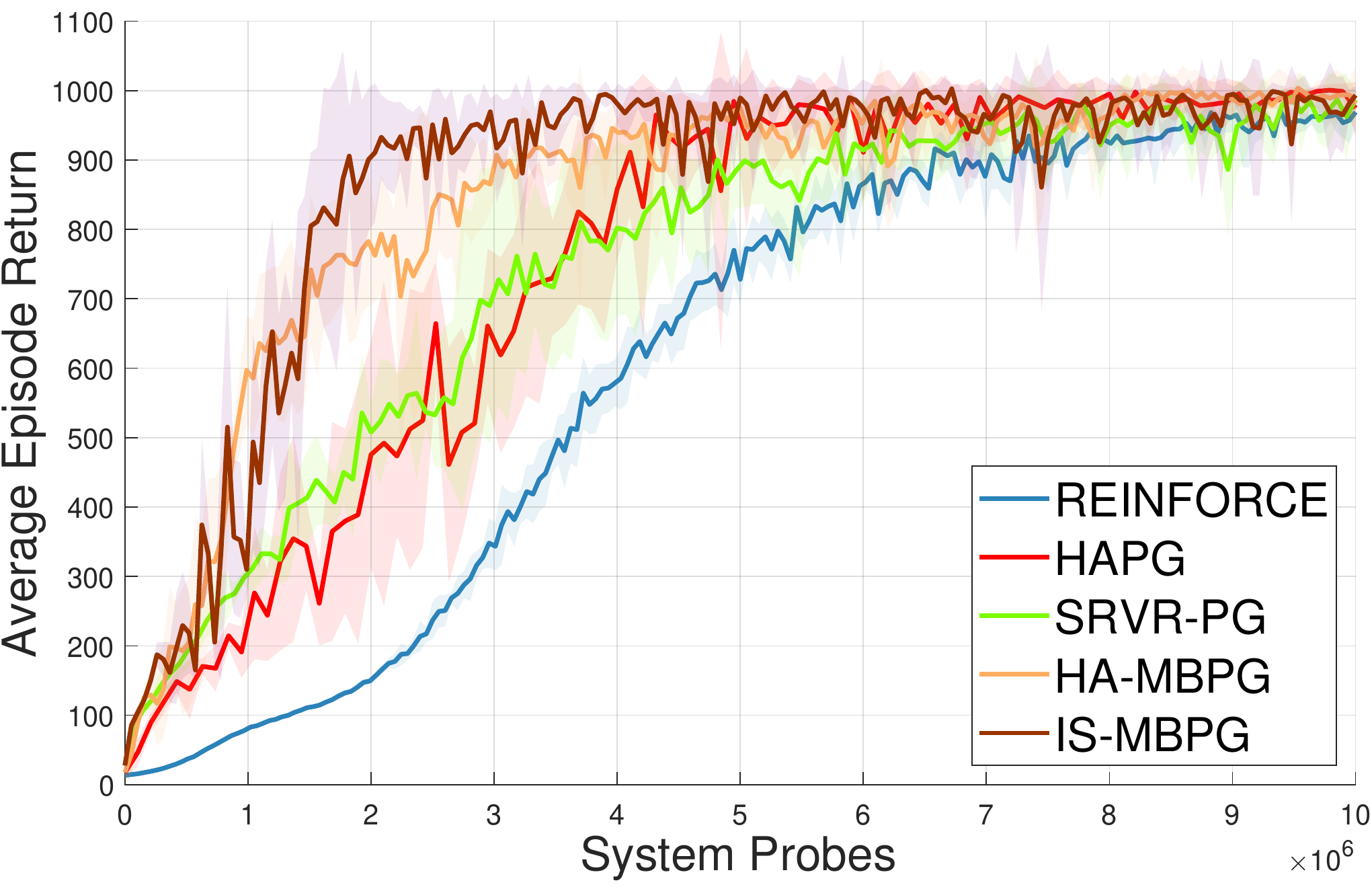}%
}
\hfil
\subfloat[HalfCheetah]{%
  \includegraphics[clip,width=0.38\textwidth]{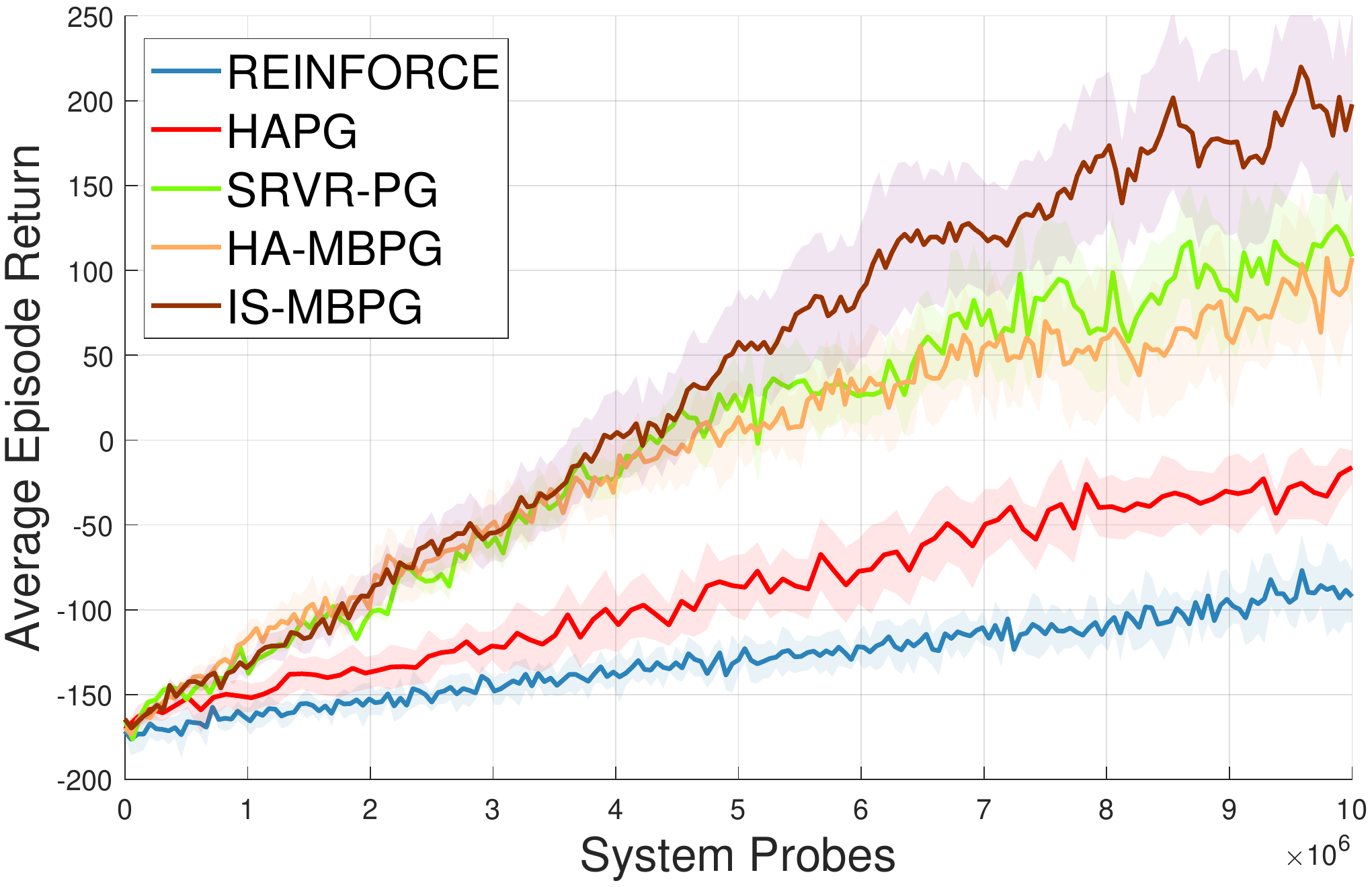}%
}
\caption{Experimental results of our algorithms (IS-MBPG and HA-MBPG) and baseline algorithms at four environments.}
\label{fig:2}
 \vspace*{-10pt}
\end{figure*}

\begin{figure*}[tb]
\centering
\subfloat[IS-MBPG]{%
  \includegraphics[clip,width=0.36\textwidth]{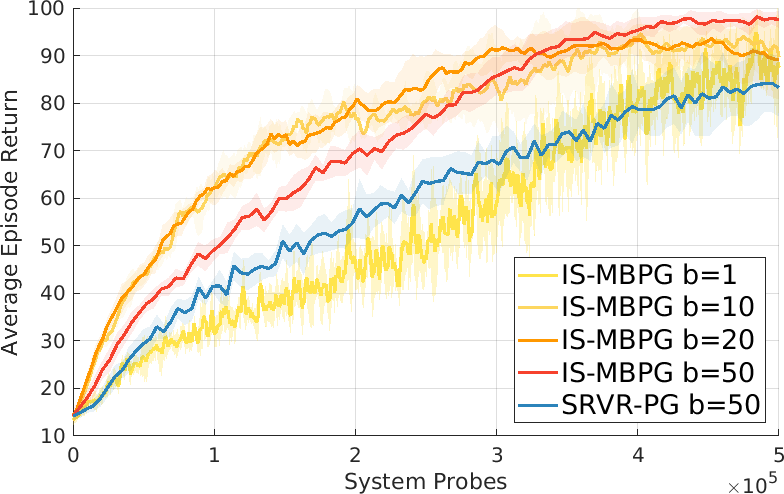}%
}
\qquad
\qquad
\subfloat[HA-MBPG]{%
  \includegraphics[clip,width=0.36\textwidth]{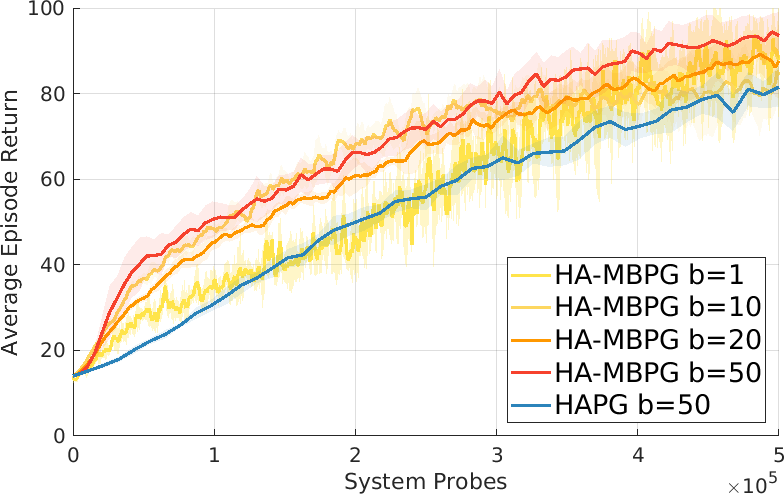}%
}
\caption{Different batch sizes for our algorithms (IS-MBPG and HA-MBPG) at CartPole environment.}
\label{fig:3}
 \vspace*{-6pt}
\end{figure*}

\begin{figure*}[tb]
\centering
\subfloat[CartPole]{%
  \includegraphics[clip,width=0.36\textwidth]{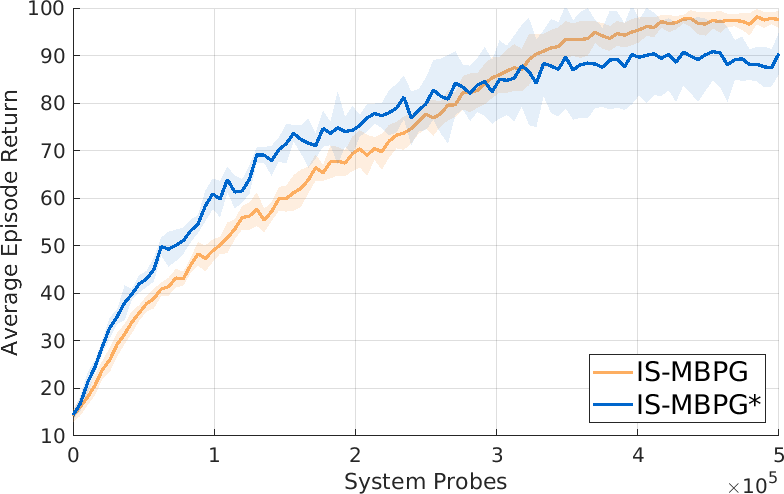}%
}
\qquad
\qquad
\subfloat[Walker]{%
  \includegraphics[clip,width=0.36\textwidth]{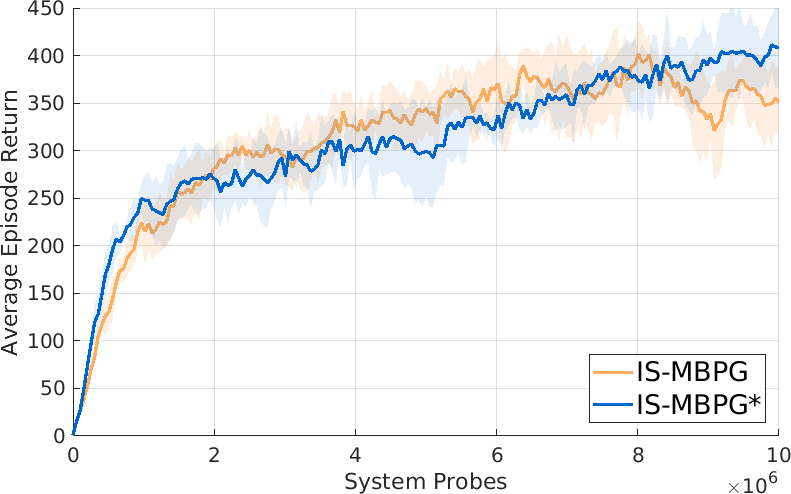}%
}
\caption{Results of IS-MBPG and IS-MBPG* algorithms at CartPole and Walker environments.}
\label{fig:4}
 \vspace*{-6pt}
\end{figure*}
\section{ Experiments }
In this section, we demonstrate the performance of our algorithms on four standard reinforcement learning tasks, 
which are CartPole, Walker, HalfCheetah and Hopper.
The first one is a discrete task from classic control, and the later three tasks are continuous RL task, which are popular MuJoCo environments~\cite{todorov2012mujoco}.
Detailed description of these environments is shown in Fig.~\ref{fig:1}. 
Our code is publicly available on https://github.com/gaosh/MBPG.
\subsection{ Experimental Setup }
In the experiment, we use Categorical Policy for CartPole, and Gaussian Policy for all the other environments.
All Policies are parameterized by the fully connected neural network.
The detail of network architecture and activation function used are shown in the Appendix A.
The network settings are similar to the HAPG~\cite{shen2019hessian} algorithm.
We implement our algorithms by using garage~\cite{garage} and pytorch~\cite{paszke2019pytorch}.
Note that Previous works mostly use environments implemented by old versions of garage,
while latest version of garage directly use environments from gym~\cite{gym}.
As a result, there might be an inconsistency of the reward calculation between this paper and previous works due to the difference of environment implementation.

In the experiments, we compare our algorithm with the existing two best algorithms: Hessian Aided Policy Gradient (HAPG)~\cite{shen2019hessian},
Stochastic Recursive Variance Reduced Policy Gradient (SRVR-PG)~\cite{xu2019sample} and a baseline algorithm: REINFORCE~\cite{sutton2000policy}.
For a fair comparison, the policies of all methods use the same initialization, which ensures that they have similar start point.
Moreover, to ease the impact of randomness, we run each method 10 times, and plot mean as well as variance interval for each of them.

In addition, for the purpose of fair comparison, we use the same batch size $\abs{\mathcal{B}}$ for all algorithms, though our algorithms do not have a requirement on it.
HAPG and SRVR-PG have sub-iterations (or inner loop), and requires additional hyper-parameters.
The inner batch size for HAPG and SRVR-PG is also set to be the same value.
For all the other hyper-parameters, we try to make them be analogous to the settings in their original paper.
One may argue that our algorithms need three hyper-parameters $k$, $m$ and $c$ to control the evolution of learning rate
while for other algorithms one hyper parameter is enough to control the learning rate.
However, it should be noticed that our algorithms do not involve any sub-iterations unlike HAPG and SRVR-PG.
Introducing sub-iterations itself naturally bring more hyper-parameters such as the number of sub-iteration and the inner batch size.
From this perspective, the hyper-parameter complexity of our algorithms resembles HAPG and SRVR-PG.
The more details of hyper-parameter selection are shown in Appendix \ref{Appendix:A}.

Similar to the HAPG algorithm, we use the \textbf{system probes} (i.e., the number of state transitions)
as the measurement of sample complexity instead of number of trajectories.
The reason of doing so is because each trajectory may have different length of states
due to a failure flag returned from the environment (often happens at the beginning of training).
Besides this reason, if using the number of trajectories as complexity measurement and the environment can return a failure flag,
a faster algorithm may have a lot more system probes given the same number of trajectories.
We also use average episode return as used in HAPG~\cite{shen2019hessian}.
\subsection{ Experimental Results }
The results of experiments are presented in Fig.~\ref{fig:2}. In the CartPole environment, our IS-MBPG and HA-MBPG algorithms have
 better performances than the other methods.
In the Walker environment, our algorithms start to have more advantages.
Specifically, the average return of IS-MBPG and HA-MBPG grows rapidly at the beginning of training.
Moreover, our IS-MBPG algorithm achieves the best final performance with a obvious margin.
HA-MBPG performs similar compared to SRVR-PG and HAPG, though it has an advantage at the beginning.
In Hopper environment, our IS-MBPG and HA-MBPG algorithms are significantly faster compared to all other methods,
while the final average reward are similar for different algorithms.
In HalfCheetah environment, IS-MBPG, HA-MBPG and SRVR-PG performs similarly at the beginning.
In the end of training, IS-MBPG can achieve the best performance.
We note that HAPG performs poorly on this task, which is probably because of the normalized gradient and fixed learning rate in their algorithm.
For all tasks, HA-MBPG are always inferior to the IS-MBPG.
One possible reason for this observation is that we use the estimated Hessian vector product instead of the exact Hessian vector product in HA-MBPG algorithm,
which brings additional estimation error to the algorithm.

In Fig.~\ref{fig:3}, we plot the average reward when changing batch size in CartPole environment.
From Fig.~\ref{fig:3}, we find that when $20\%$ of the original batch size, our HA-MBPG and IS-MBPG algorithms still outperform
the HAPG and SRVR-PG algorithms, respectively. When the batch size is 1, our HA-MBPG and IS-MBPG algorithms still
reach a good performance.
These results demonstrate that our HA-MBPG and IS-MBPG algorithms are not sensitive to the selection of batch size.
Fig.~\ref{fig:4} shows that the non-adaptive IS-MBPG* algorithm also has similar performances as the adaptive IS-MBPG algorithm.
\section{Conclusion}
 \vspace*{-6pt}
In the paper, we proposed a class of efficient momentum-based policy gradient methods (i.e., IS-MBPG and HA-MBPG),
which use adaptive learning rates and do not require any large batches.
Moreover, we proved that both IS-MBPG and HA-MBPG methods reach
the best known sample complexity of $O(\epsilon^{-3})$, which only require one trajectory at each iteration.
In particular, we also presented a non-adaptive version of IS-MBPG method (i.e., IS-MBPG*),
which has a simple monotonically decreasing learning rate.
We proved that the IS-MBPG* also reaches the best known sample complexity of $O(\epsilon^{-3})$ only required one trajectory at each iteration.

\section*{Acknowledgements}
We thank the anonymous reviewers for their helpful comments. We also thank the IT Help Desk at University of Pittsburgh. This work was partially supported by U.S. NSF
IIS 1836945, IIS 1836938, IIS 1845666, IIS
1852606, IIS 1838627, IIS 1837956.


\nocite{langley00}

\bibliography{MBPG}
\bibliographystyle{icml2020}

%
%
%


\begin{onecolumn}

\begin{appendices}
\begin{table*}[]
    \centering
    \begin{tabular}{c|c|c|c|c}
    \hline
         Environments& CartPole & Walker & Hopper & HalfCheetah\\
         \hline
         \hline
         Horizon & 100 & 500  & 1000 & 500\\
         Baseline & None & Linear & Linear & Linear\\
         Neural Network (NN) sizes & $8\times8$ & $64\times64$  & $64\times64 $ &  $64\times64$\\
         NN activation function & Tanh & Tanh  & Tanh &  Tanh \\
         Number of timesteps & $5\times10^5$ &$1\times10^7$  & $1\times10^7$ &$1\times10^7$ \\
         Batch size $\abs{\mathcal{B}}$ & 50 & 100 & 50 & 100\\
         HAPG $\abs{\mathcal{B}_{\text{sub}}}$ & 10 & 10 & 10 & 10\\
         SRVR-PG $\abs{\mathcal{B}_{\text{sub}}}$ & 10 & 10 & 10 & 10\\
         HAPG $n_{\text{sub}}$ & 5 &10 & 10 &10\\
         SRVR-PG $n_{\text{sub}}$ & 3 & 2 & 2 & 2\\
         IS-MBPG/HA-MBPG $k$ & 0.75 & 0.75 & 0.75 & 0.75 \\
         IS-MBPG/HA-MBPG $c$ & 2 & 2 & 1 & 1\\
         IS-MBPG/HA-MBPG $m$ & 2 & 12 & 3 & 3\\
         IS-MBPG* $k$ & 0.9 & 0.9 &- &- \\
         IS-MBPG* $c$ & 2 & 2 &- & -\\
         IS-MBPG* $m$ & 2 & 12 &- & -\\
         REINFORCE learning rate $\eta$  & 0.01 & 0.01 & 0.01 & 0.01\\
         HAPG learning rate $\eta$ & 0.01 & 0.01 & 0.01 & 0.01\\
         SRVR-PG learning rate $\eta$ & 0.1 & 0.1 & 0.1 & 0.1\\
         \hline
    \end{tabular}
    \caption{Hyper-parameter Details. $\abs{\mathcal{B}_{\text{sub}}}$ represents sub-iteration (inner-loop) batch size. $n_{\text{sub}}$ represents number of sub-iterations, which is called $p$ in the original paper of HAPG
    and $m$ in the oringinal paper of SRVR-PG. Although the learning rate $\eta$ of HAPG is given a small value $\eta=\epsilon$ in the theoretical analysis,
    we choose learning rate $\eta=0.01$ as given in the experiments of~\cite{shen2019hessian}. }
    \label{tab:2}
\end{table*}
\section{ Supplementary Materials for ``Momentum-Based Policy Gradient Methods" }
\label{Appendix:A}
In this section, we first provide the details of hyper-parameter selection for the algorithms in Table \ref{tab:2}.
Table \ref{tab:2} also shows that the detail of network architecture and activation function used in the experiments.
Next, we study the convergence properties of our algorithms.
We begin with giving some useful lemmas.
\begin{lemma} \label{lem:A1}
 (Lemma 1 in \cite{cortes2010learning}) Let $w(x)=P(x)/Q(x)$ be the importance weight for distributions $P$ and $Q$.
 The following identities hold for the expectation, second moment, and variance of $w(x)$
 \begin{align}
  \mathbb{E}[w(x)] = 1, \ \mathbb{E}[w^2(x)]= d_2(P||Q), \ \mathbb{V}[w(x)] = d_2(P||Q)-1,
 \end{align}
 where $d_2(P||Q)=2^{D(P||Q)}$, and $D(P||Q)$ is the $R\acute{e}nyi$ divergence between distributions $P$ and $Q$.
\end{lemma}
\begin{lemma} \label{lem:A2}
 Under Assumptions 1 and 3, let $w(\tau|\theta_{t-1},\theta_t)= g(\tau|\theta_{t-1})/g(\tau | \theta_{t})$, we have
 \begin{align}
  \mathbb{V}[w(\tau|\theta_{t-1},\theta_t)] \leq C^2_w\|\theta_t-\theta_{t-1}\|^2,
 \end{align}
 where $C_w = \sqrt{H(2HM_g^2+M_h)(W+1)}$.
\end{lemma}
 \vspace*{-6pt}
\begin{proof}
 This proof can easy follow the proof of Lemma 6.1 in \cite{xu2019improved}.
\end{proof}

\begin{lemma} \label{lem:A3}
Under Assumption 1, let $e_t = \nabla J(\theta_t)-u_t$. Given $0< \eta_t \leq \frac{1}{2L}$ for all $t\geq 1$,
we have
\begin{align}
\mathbb{E} [J(\theta_{t+1})] \geq \mathbb{E} [J(\theta_{t}) - \frac{3\eta_t}{4}\|e_t\|^2 + \frac{\eta_t}{8}\|\nabla J(\theta_t)\|^2 ].
\end{align}
\end{lemma}
 \vspace*{-6pt}
\begin{proof}
Let $e_t = \nabla J(\theta_t)-u_t$. By using $J(\theta)$ is $L$-smooth, we have
\begin{align}
 \mathbb{E} [J(\theta_{t+1})] & \geq  \mathbb{E} [J(\theta_{t}) + \nabla J(\theta_{t})^T(\theta_{t+1}-\theta_{t}) - \frac{L}{2}\|\theta_{t+1}-\theta_{t}\|^2] = \mathbb{E} [J(\theta_{t}) + \eta_t\nabla J(\theta_{t})^Tu_t - \frac{L\eta_t^2}{2}\|u_{t}\|^2] \nonumber \\
 & = \mathbb{E} [J(\theta_{t}) + \eta_t(\nabla J(\theta_{t})-u_t)^Tu_t + \eta_t\|u_t\|^2 - \frac{L\eta_t^2}{2}\|u_{t}\|^2] \nonumber \\
 & \geq  \mathbb{E} [J(\theta_{t}) - \frac{\eta_t}{2}\|\nabla J(\theta_{t})-u_t\|^2 + \frac{\eta_t}{2}(1-L\eta_t)\|u_t\|^2] \nonumber \\
 & \geq \mathbb{E} [J(\theta_{t}) - \frac{\eta_t}{2}\|\nabla J(\theta_{t})-u_t\|^2 + \frac{\eta_t}{4}\|u_t\|^2] \nonumber \\
 & \geq \mathbb{E} [J(\theta_{t}) - \frac{3\eta_t}{4}\|e_t\|^2 + \frac{\eta_t}{8}\|\nabla J(\theta_t)\|^2 ],
\end{align}
where the second inequality holds by Young's inequality, and the third inequality holds by $0< \eta_t \leq \frac{1}{2L}$, and
the last inequality follows by $\|\nabla J(\theta_t)\|^2 \leq 2\|u_t\|^2 + 2\|\nabla J(\theta_t)-u_t\|^2 = 2\|u_t\|^2 + 2\|e_t\|^2$.
\end{proof}
 \vspace*{-6pt}
\subsection{ Convergence Analysis of IS-MBPG Algorithm }
\label{Appendix:A1}
In this subsection, we analyze the convergence properties of IS-MBPG algorithm.
For notational simplicity, let $\nabla J(\theta)$ denote $\nabla_{\theta} J(\theta)$.

\begin{lemma} \label{lem:A4}
 Assume that the stochastic policy gradient $u_t$ be generated from Algorithm \ref{alg:1}, and let $e_t = u_t-\nabla J(\theta_t)$,  we have
 \begin{align}
\mathbb{E}\big[ \eta^{-1}_{t-1}\|e_t\|^2 \big]
\leq 2\beta^2_t\eta^{-1}_{t-1}G_t^2 + \eta^{-1}_{t-1}(1-\beta_t)^2\big(1+8\eta^2_{t-1}B^2\big)\mathbb{E} \|e_{t-1}\|^2 + 8(1-\beta_t)^2B^2\eta_{t-1}\|\nabla J(\theta_{t-1})\|^2, \nonumber
 \end{align}
 where $B^2 = L^2 + 2G^2C^2_w$ with $C_w = \sqrt{H(2HM_g^2+M_h)(W+1)}$.
\end{lemma}
\begin{proof}
 By the definition of $u_t$ in Algorithm \ref{alg:1}, we have
  \begin{align}
  u_t - u_{t-1} = -\beta_tu_{t-1}+ \beta_t g(\tau_t|\theta_t) + (1-\beta_t)\big( g(\tau_t | \theta_{t})- w(\tau_t|\theta_{t-1},\theta_t) g(\tau_t|\theta_{t-1})\big).
 \end{align}
 Then we have
 \begin{align} \label{eq:A1}
 \mathbb{E}\big[ \eta^{-1}_{t-1} \|e_t\|^2\big]& = \mathbb{E}\big[\eta^{-1}_{t-1} \|\nabla J(\theta_{t-1}) - u_{t-1} + \nabla J(\theta_t)-\nabla J(\theta_{t-1}) -(u_t-u_{t-1})\|^2\big] \\
  & =  \mathbb{E}\big[ \eta^{-1}_{t-1} \|\nabla J(\theta_{t-1}) - u_{t-1} + \nabla J(\theta_t)-\nabla J(\theta_{t-1}) + \beta_tu_{t-1}- \beta_t g(\tau_t|\theta_t)\nonumber \\
  & \quad - (1-\beta_t)\big( g(\tau_t | \theta_{t})- w(\tau_t|\theta_{t-1},\theta_t) g(\tau_t|\theta_{t-1})\big)\|^2\big] \nonumber \\
  & = \mathbb{E}\big[ \eta^{-1}_{t-1} \|(1-\beta_t)(\nabla J(\theta_{t-1}) - u_{t-1}) + \beta_t(\nabla J(\theta_t)- g(\tau_t|\theta_t))\nonumber \\
  & \quad - (1-\beta_t)\big( g(\tau_t | \theta_{t})- w(\tau_t|\theta_{t-1},\theta_t) g(\tau_t|\theta_{t-1})-(\nabla J(\theta_t)-\nabla J(\theta_{t-1}))\big)\|^2\big] \nonumber \\
  & = \eta^{-1}_{t-1}(1-\beta_t)^2\mathbb{E} \|\nabla J(\theta_{t-1}) - u_{t-1}\|^2 + \eta^{-1}_{t-1}\mathbb{E}\big[\|\beta_t(\nabla J(\theta_t)- g(\tau_t|\theta_t))\nonumber \\
  & \quad - (1-\beta_t)\big( g(\tau_t | \theta_{t})- w(\tau_t|\theta_{t-1},\theta_t) g(\tau_t|\theta_{t-1})-(\nabla J(\theta_t)-\nabla J(\theta_{t-1}))\big)\|^2\big] \nonumber \\
  & \leq \eta^{-1}_{t-1}(1-\beta_t)^2\mathbb{E} \|\nabla J(\theta_{t-1}) - u_{t-1}\|^2 + 2\beta^2_t\eta^{-1}_{t-1}\mathbb{E}\|\nabla J(\theta_t)- g(\tau_t|\theta_t)\|^2\nonumber \\
  & \quad +2(1-\beta_t)^2 \mathbb{E}\| g(\tau_t | \theta_{t})- w(\tau_t|\theta_{t-1},\theta_t) g(\tau_t|\theta_{t-1})-(\nabla J(\theta_t)-\nabla J(\theta_{t-1}))\|^2 \nonumber \\
  & \leq \eta^{-1}_{t-1}(1-\beta_t)^2\mathbb{E} \|e_{t-1}\|^2 + 2\beta^2_t\eta^{-1}_{t-1}\mathbb{E}\|g(\tau_t|\theta_t)\|^2
  + 2(1-\beta_t)^2\eta^{-1}_{t-1} \mathbb{E}\| g(\tau_t | \theta_{t})- w(\tau_t|\theta_{t-1},\theta_t) g(\tau_t|\theta_{t-1})\|^2 \nonumber \\
  & = \eta^{-1}_{t-1}(1-\beta_t)^2\mathbb{E} \|e_{t-1}\|^2 + 2\beta^2_t\eta^{-1}_{t-1}G_t^2 + 2(1-\beta_t)^2\eta^{-1}_{t-1}
  \underbrace{ \mathbb{E}\| g(\tau_t | \theta_{t})- w(\tau_t|\theta_{t-1},\theta_t) g(\tau_t|\theta_{t-1})\|^2 }_{=T_1}, \nonumber
 \end{align}
 where the forth equality holds by $\mathbb{E}_{\tau_t\sim p(\tau|\theta_t)}[g(\tau_t|\theta_t)]=\nabla J(\theta_t)$ and $\mathbb{E}_{\tau_t\sim p(\tau|\theta_t)}[g(\tau_t|\theta_t)- w(\tau_t|\theta_{t-1},\theta_t) g(\tau_t|\theta_{t-1})]=\nabla J(\theta_t)-\nabla J(\theta_{t-1})$; the first inequality follows by Young's inequality; and the last inequality holds by $\mathbb{E}\|\zeta-\mathbb{E}[\zeta]\|^2=\mathbb{E}\|\zeta\|^2-(\mathbb{E}[\zeta])^2 \leq \mathbb{E}\|\zeta\|^2$.

 Next, we give an upper bound of the term $T_1$ as follows:
 \begin{align} \label{eq:A2}
  T_1 & = \mathbb{E}\| g(\tau_t | \theta_{t})- w(\tau_t|\theta_{t-1},\theta_t) g(\tau_t|\theta_{t-1})\|^2 \nonumber \\
  & = \mathbb{E}\| g(\tau_t | \theta_{t})-g(\tau_t|\theta_{t-1}) + g(\tau_t|\theta_{t-1})- w(\tau_t|\theta_{t-1},\theta_t) g(\tau_t|\theta_{t-1})\|^2 \nonumber \\
  & \leq 2\mathbb{E}\| g(\tau_t | \theta_{t})-g(\tau_t|\theta_{t-1})\|^2 + 2\mathbb{E}\|(1- w(\tau_t|\theta_{t-1},\theta_t)) g(\tau_t|\theta_{t-1})\|^2 \nonumber \\
  & \leq 2L^2\|\theta_t - \theta_{t-1}\|^2 + 2G^2\mathbb{E}\|1- w(\tau_t|\theta_{t-1},\theta_t)\|^2 \nonumber \\
  & = 2L^2\|\theta_t - \theta_{t-1}\|^2 + 2G^2\mathbb{V}\big(w(\tau_t|\theta_{t-1},\theta_t)\big) \nonumber \\
  & \leq  2(L^2 + 2G^2C^2_w)\|\theta_t-\theta_{t-1}\|^2,
 \end{align}
 where the second inequality holds by Proposition \ref{pro:1}, and the third equality holds by Lemma \ref{lem:A1},
 and the last inequality follows by Lemma \ref{lem:A2}.

 Combining the inequalities \eqref{eq:A1} with \eqref{eq:A2}, let $B^2 = L^2 + 2G^2C^2_w$, we have
 \begin{align}
  \mathbb{E}\big[ \eta^{-1}_{t-1} \|e_t\|^2\big]& \leq \eta^{-1}_{t-1}(1-\beta_t)^2\mathbb{E} \|e_{t-1}\|^2 + 2\beta^2_t\eta^{-1}_{t-1}G_t^2+ 4(1-\beta_t)^2\eta^{-1}_{t-1}B^2\|\theta_t-\theta_{t-1}\|^2 \nonumber \\
  & = \eta^{-1}_{t-1}(1-\beta_t)^2\mathbb{E} \|e_{t-1}\|^2 + 2\beta^2_t\eta^{-1}_{t-1}G_t^2 + 4(1-\beta_t)^2B^2\eta_{t-1}\|u_{t-1}\|^2 \nonumber \\
  & = \eta^{-1}_{t-1}(1-\beta_t)^2\mathbb{E} \|e_{t-1}\|^2 + 2\beta^2_t\eta^{-1}_{t-1}G_t^2 + 4(1-\beta_t)^2B^2\eta_{t-1}\|e_{t-1} + \nabla J(\theta_{t-1})\|^2 \nonumber \\
  & \leq 2\beta^2_t\eta^{-1}_{t-1}G_t^2 + \eta^{-1}_{t-1}(1-\beta_t)^2\big(1+8\eta^2_{t-1}B^2\big)\mathbb{E} \|e_{t-1}\|^2 + 8(1-\beta_t)^2B^2\eta_{t-1}\|\nabla J(\theta_{t-1})\|^2.
 \end{align}

\end{proof}

\begin{theorem} \label{th:A1}
Assume that the sequence $\{\theta_t\}_{t=1}^T$ be generated from Algorithm \ref{alg:1}, and let $B^2 = L^2 + 2G^2C^2_w$, $k=O(\frac{G^{2/3}}{L})$ $c=\frac{G^2}{3k^3L}+104B^2$,
$m = \max\{2G^2,(2Lk)^3,(\frac{ck}{2L})^3\}$ and $\eta_0 = \frac{k}{m^{1/3}}$, we have
\begin{align}
  \mathbb{E}\|\nabla J(\theta_\zeta)\|=\frac{1}{T}\sum_{t=1}^T\mathbb{E}\|\nabla J(\theta_t)\| \leq \frac{\sqrt{2\Omega}m^{1/6} + 2\Omega^{3/4}}{\sqrt{T}} + \frac{2\sqrt{\Omega}\sigma^{1/3}}{T^{1/3}}, \nonumber
\end{align}
 where $\Omega=\frac{1}{k}\big(16(J^* - J(\theta_1)) + \frac{m^{1/3}}{8B^2k}\sigma^2 + \frac{c^2k^{3}}{4B^2}\ln(T+2)\big)$ with $J^*=\sup_{\theta}J(\theta) <+\infty$.
\end{theorem}
\begin{proof}
Due to $m\geq (2Lk)^3$, we have $\eta_t\leq \frac{k}{m^{1/3}}\leq \frac{1}{2L}$. Since $\eta_t\leq \frac{1}{2L}$ and
$m\geq (\frac{ck}{2L})^3$, we have $\beta_{t+1}=c\eta^2_t\leq \frac{c\eta_t}{2L}\leq \frac{ck}{2Lm^{1/3}}\leq 1$.
By Lemma \ref{lem:A4}, we have
 \begin{align} \label{eq:A4}
  \mathbb{E}[\eta^{-1}_{t}\|e_{t+1}\|^2-\eta^{-1}_{t-1}\|e_t\|^2] & \leq \mathbb{E}\big[ 2\beta^2_{t+1}\eta^{-1}_{t}G_{t+1}^2 + \eta^{-1}_{t}(1-\beta_{t+1})^2\big(1+8\eta^2_{t}B^2\big) \|e_{t}\|^2 \nonumber \\
  & \quad + 8(1-\beta_{t+1})^2B^2\eta_{t}\|\nabla J(\theta_{t})\|^2 -\eta^{-1}_{t-1}\|e_t\|^2 \big] \nonumber \\
  & \leq  \mathbb{E}\big[ 2c^2\eta^{3}_{t}G_{t+1}^2 + \underbrace{ \big( \eta^{-1}_{t}(1-\beta_{t+1})(1+8\eta^2_{t}B^2) -\eta^{-1}_{t-1} \big) \|e_{t}\|^2 }_{=T_2}
  + 8B^2\eta_{t}\|\nabla J(\theta_{t})\|^2 \big],
 \end{align}
where the last inequality holds by $0< \beta_{t+1} \leq 1$.
Since the function $x^{1/3}$ is cancave, we have $(x+y)^{1/3}\leq x^{1/3} + yx^{-2/3}/3$.
Then we have
 \begin{align}
  \eta^{-1}_{t} - \eta^{-1}_{t-1}  &= \frac{1}{k}\bigg( \big(m+\sum_{i=1}^tG^2_i\big)^{1/3} - \big(m+\sum_{i=1}^{t-1}G^2_i\big)^{1/3} \bigg) \leq \frac{G^2_t}{3k(m+\sum_{i=1}^{t-1}G^2_i)^{2/3}} \nonumber \\
  & \leq \frac{G^2_t}{3k(m - G^2+\sum_{i=1}^{t}G^2_i)^{2/3}} \leq \frac{G^2_t}{3k(m/2 +\sum_{i=1}^{t}G^2_i)^{2/3}} \leq \frac{2^{2/3}G^2_t}{3k(m +\sum_{i=1}^{t}G^2_i)^{2/3}} \nonumber \\
  & \leq \frac{2^{2/3}G^2}{3k^3}\eta_t^2 \leq \frac{2^{2/3}G^2}{6k^3L}\eta_t\leq \frac{G^2}{3k^3L}\eta_t,
 \end{align}
where the third inequality holds by $m\geq 2G^2$, and the sixth inequality holds by $0<\eta\leq \frac{1}{2L}$.

Next, considering the upper bound of the term $T_2$, we have
 \begin{align} \label{eq:A5}
  T_2 &= \big( \eta^{-1}_{t}(1-\beta_{t+1})(1+8\eta^2_{t}B^2) -\eta^{-1}_{t-1} \big) \|e_{t}\|^2 \nonumber \\
  & = \big( \eta^{-1}_{t} - \eta^{-1}_{t-1} + 8B^2\eta_t - \beta_{t+1}\eta^{-1}_{t} - 8\eta_t\beta_{t+1}B^2 \big) \|e_{t}\|^2 \nonumber \\
  & \leq \big( \eta^{-1}_{t} - \eta^{-1}_{t-1} + 8B^2\eta_t - \beta_{t+1}\eta^{-1}_{t} \big) \|e_{t}\|^2 \nonumber \\
  & \leq  \big( \frac{G^2}{3k^3L}\eta_t + 8B^2\eta_t - c\eta_{t} \big) \|e_{t}\|^2= -96B^2\eta_t\|e_{t}\|^2,
 \end{align}
 where the last equality holds by $c=\frac{G^2}{3k^3L}+104B^2$. Combining the inequalities \eqref{eq:A4} with \eqref{eq:A5}, we have
 \begin{align} \label{eq:A6}
  \mathbb{E}[\eta^{-1}_{t}\|e_{t+1}\|^2-\eta^{-1}_{t-1}\|e_t\|^2] \leq  \mathbb{E}\big[ 2c^2\eta^{3}_{t}G_{t+1}^2 -96B^2\eta_t\|e_{t}\|^2
  + 8B^2\eta_{t}\|\nabla J(\theta_{t})\|^2 \big].
 \end{align}

We define a \emph{Lyapunov} function $\Phi_t = J(\theta_t) - \frac{1}{128B^2\eta_{t-1}}\|e_t\|^2$ for any $t\geq 1$.
Then we have
 \begin{align} \label{eq:A7}
  \mathbb{E}[\Phi_{t+1} - \Phi_t] & = \mathbb{E}\big[ J(\theta_{t+1}) - J(\theta_t) - \frac{1}{128B^2\eta_{t}} \|e_{t+1}\|^2 + \frac{1}{128B^2\eta_{t-1}}\|e_t\|^2 \big] \nonumber \\
  & \geq \mathbb{E}\big[- \frac{3\eta_t}{4}\|e_t\|^2 + \frac{\eta_t}{8}\|\nabla J(\theta_t)\|^2 - \frac{1}{128B^2}(\eta^{-1}_{t}\|e_{t+1}\|^2 - \eta^{-1}_{t-1}\|e_t\|^2)\big] \nonumber \\
  & \geq -\frac{c^2\eta^{3}_{t}G_{t+1}^2 }{64B^2}+ \frac{\eta_t}{16}\mathbb{E}\|\nabla J(\theta_t)\|^2,
 \end{align}
where the first inequality holds by the Lemma \ref{lem:A3}, and the second inequality follows by the above inequality \eqref{eq:A6}.
Summing the above inequality \eqref{eq:A7} over $t$ from $1$ to $T$, we obtain
\begin{align} \label{eq:A8}
 \sum_{t=1}^T\mathbb{E}[\eta_t\|\nabla J(\theta_t)\|^2] & \leq \mathbb{E}[16(\Phi_{T+1} - \Phi_1)] + \sum_{t=1}^T\frac{c^2\eta^{3}_{t}G_{t+1}^2 }{4B^2} \nonumber \\
 & \leq \mathbb{E}[16(J^* - J(\theta_1))] + \frac{1}{8B^2\eta_0}\mathbb{E}\|e_1\|^2 + \frac{c^2k^{3}}{4B^2}\sum_{t=1}^T\frac{G_{t+1}^2 }{m+\sum_{i=1}^tG_i^2} \nonumber \\
 & \leq \mathbb{E}[16(J^* - J(\theta_1))] + \frac{1}{8B^2\eta_0}\mathbb{E}\|e_1\|^2 + \frac{c^2k^{3}}{4B^2}\sum_{t=1}^T\frac{G_{t+1}^2 }{G^2+\sum_{i=1}^{t+1}G_i^2}  \nonumber \\
 & \leq \mathbb{E}[16(J^* - J(\theta_1))] + \frac{1}{8B^2\eta_0}\mathbb{E}\|e_1\|^2 + \frac{c^2k^{3}}{4B^2}\sum_{t=1}^T\big(\ln(G^2+ \sum_{i=1}^{t+1}G_i^2) - \ln(G^2+ \sum_{i=1}^{t}G_i^2)\big)  \nonumber \\
 & \leq \mathbb{E}[16(J^* - J(\theta_1))] + \frac{1}{8B^2\eta_0}\mathbb{E}\|e_1\|^2 + \frac{c^2k^{3}}{4B^2}(\ln(G^2+ \sum_{i=1}^{T+1}G_i^2) - \ln(G^2)) \nonumber \\
 & \leq \mathbb{E}[16(J^* - J(\theta_1))] + \frac{m^{1/3}}{8B^2k}\sigma^2 + \frac{c^2k^{3}}{4B^2}\ln(1+ \sum_{i=1}^{T+1}\frac{G_i^2}{G^2}) \nonumber \\
 & \leq \mathbb{E}[16(J^* - J(\theta_1))] + \frac{m^{1/3}}{8B^2k}\sigma^2 + \frac{c^2k^{3}}{4B^2}\ln(T+2),
\end{align}
where $J^*=\sup_{\theta}J(\theta) <+\infty$, and the fourth inequality holds by the concavity of the function $\ln(x)$, and the sixth inequality holds by the definition of $e_1$ and $\eta_0$.

By Cauchy-Schwarz inequality, we have $\mathbb{E}[XY]^2 \leq\mathbb{E}[X^2]\mathbb{E}[Y^2]$. Let $X=\sqrt{\eta_T\sum_{t=1}^{T-1}\|\nabla J(\theta_t)\|^2}$ and $Y=\sqrt{1/\eta_T}$,
we have
\begin{align}
 \mathbb{E}\big[\sum_{t=1}^T\|\nabla J(\theta_t)\|^2\big] \leq \mathbb{E}[1/\eta_T]\mathbb{E}\big[ \eta_T\sum_{t=1}^T\|\nabla J(\theta_t)\|^2\big].
\end{align}
Since $\eta_t$ is decreasing, we have
\begin{align} \label{eq:A9}
  \mathbb{E}\big[\sum_{t=1}^T\|\nabla J(\theta_t)\|^2\big] \leq \mathbb{E}[1/\eta_T]\mathbb{E}\big[\sum_{t=1}^T\eta_T\|\nabla J(\theta_t)\|^2\big] \leq \mathbb{E}[1/\eta_T]\mathbb{E}\big[\sum_{t=1}^T \eta_t\|\nabla J(\theta_t)\|^2\big].
\end{align}
Combining the inequalities \eqref{eq:A8} and \eqref{eq:A9}, we obtain
\begin{align} \label{eq:A10}
 \mathbb{E}\big[\sum_{t=1}^T\|\nabla J(\theta_t)\|^2\big] &\leq \mathbb{E}\big[ \frac{16(J^* - J(\theta_1)) + \frac{m^{1/3}}{8B^2k}\sigma^2 + \frac{c^2k^{3}}{4B^2}\ln(T+2)}{\eta_T} \big] \nonumber \\
 & = \mathbb{E}\big[ \Omega\big(m+\sum_{t=1}^TG^2_t\big)^{1/3} \big]
\end{align}
where $\Omega=\frac{1}{k}\big(16(J^* - J(\theta_1)) + \frac{m^{1/3}}{8B^2k}\sigma^2 + \frac{c^2k^{3}}{4B^2}\ln(T+2)\big)$.

By Assumption 2, we have $G^2_t = \|g(\tau|\theta_t)-\nabla J(\theta_t)+\nabla J(\theta_t)\|^2 \leq 2\|g(\tau|\theta_t)-\nabla J(\theta_t)\|^2 + 2\|\nabla J(\theta_t)\|^2 \leq 2\sigma^2+2\|\nabla J(\theta_t)\|^2$.
Then using the inequality $(a+b)^{1/3}\leq a^{1/3} + b^{1/3}$ for all $a,b>0$ to the inequality \eqref{eq:A10}, we obtain
\begin{align}
 \bigg(\mathbb{E}\sqrt{\sum_{t=1}^T\|\nabla J(\theta_t)\|^2}\bigg)^2 & \leq \mathbb{E}\bigg[\sqrt{\sum_{t=1}^T\|\nabla J(\theta_t)\|^2}\bigg]^2 = \mathbb{E}\big[\sum_{t=1}^T\|\nabla J(\theta_t)\|^2\big] \nonumber \\
 & \leq \mathbb{E}\bigg[\Omega(m+2T\sigma^2)^{1/3} + 2^{1/3}\Omega\big(\sum_{t=1}^T\|\nabla J(\theta_t)\|^2\big)^{1/3}\bigg] \nonumber \\
 & =\Omega(m+2T\sigma^2)^{1/3} + 2^{1/3}\Omega\mathbb{E}\bigg[\sqrt{\sum_{t=1}^T\|\nabla J(\theta_t)\|^2}\bigg]^{2/3} \nonumber \\
 & \leq \Omega(m+2T\sigma^2)^{1/3} + 2^{1/3}\Omega\bigg[\mathbb{E}\sqrt{\sum_{t=1}^T\|\nabla J(\theta_t)\|^2}\bigg]^{2/3},
\end{align}
where the first inequality holds by the convexity of the function $x^2$, and the last inequality holds by the concavity of the function $x^{2/3}$.
For simplicity, let $Z=\sqrt{\sum_{t=1}^T\|\nabla J(\theta_t)\|^2}$, we have
\begin{align} \label{eq:A11}
 \big(\mathbb{E}[Z]\big)^2 \leq \Omega(m+2T\sigma^2)^{1/3} + 2^{1/3}\Omega\big(\mathbb{E}[Z]\big)^{2/3}.
\end{align}
The inequality \eqref{eq:A11} implies that $ \big(\mathbb{E}[Z]\big)^2 \leq 2\Omega(m+2T\sigma^2)^{1/3}$ or $ \big(\mathbb{E}[Z]\big)^2 \leq 2\cdot2^{1/3}\Omega\big(\mathbb{E}[Z]\big)^{2/3}$.
Thus, we have
\begin{align}
 \mathbb{E}[Z] \leq \sqrt{2\Omega}(m+2T\sigma^2)^{1/6} + 2\Omega^{3/4}.
\end{align}
By Cauchy-Schwarz inequality, then we have
\begin{align}
  \frac{1}{T}\sum_{t=1}^T\mathbb{E}\|\nabla J(\theta_t)\| &= \mathbb{E} \big[\frac{1}{T}\sum_{t=1}^T\|\nabla J(\theta_t)\|\big] \leq \mathbb{E}\bigg[\sqrt{\frac{1}{T}}\sqrt{\sum_{t=1}^T\|\nabla J(\theta_t)\|^2}\bigg] \nonumber \\
  & \leq \frac{\sqrt{2\Omega}(m+2T\sigma^2)^{1/6} + 2\Omega^{3/4}}{\sqrt{T}} \nonumber \\
  & \leq \frac{\sqrt{2\Omega}m^{1/6} + 2\Omega^{3/4}}{\sqrt{T}} + \frac{2\sqrt{\Omega}\sigma^{1/3}}{T^{1/3}},
\end{align}
where the last inequality follows by the inequality $(a+b)^{1/6}\leq a^{1/6} + b^{1/6}$ for all $a,b>0$.
\end{proof}

\subsection{ Convergence Analysis of HA-MBPG Algorithm  }
\label{Appendix:A2}
In this subsection, we analyze the convergence properties of HA-MBPG algorithm.

\begin{lemma} \label{lem:B1}
 Assume that the stochastic policy gradient $u_t$ be generated from Algorithm \ref{alg:2}. Let $e_t = u_t-\nabla J(\theta_t)$,
 we have
 \begin{align}
\mathbb{E}\big[ \eta^{-1}_{t-1}\|e_t\|^2 \big]
\leq 4(W+1)\beta^2_t\eta^{-1}_{t-1}G_t^2 + \eta^{-1}_{t-1}(1-\beta_t)^2\big(1+4\eta^2_{t-1}L^2\big)\mathbb{E} \|e_{t-1}\|^2 + 4(1-\beta_t)^2L^2\eta_{t-1}\|\nabla J(\theta_{t-1})\|^2. \nonumber
 \end{align}
\end{lemma}
\begin{proof}
 By the definition of $u_t$ in Algorithm \ref{alg:2}, we have
 \begin{align}
  u_t - u_{t-1} = -\beta_tu_{t-1}+ \beta_t w(\tau_t|\theta_{t},\theta_t(\alpha))g(\tau_t|\theta_t) + (1-\beta_t) \Delta_t.
 \end{align}
 Then we have
 \begin{align} \label{eq:B1}
 \mathbb{E}\big[ \eta^{-1}_{t-1} \|e_t\|^2\big]& = \mathbb{E}\big[\eta^{-1}_{t-1} \|\nabla J(\theta_{t-1}) - u_{t-1} + \nabla J(\theta_t)-\nabla J(\theta_{t-1}) -(u_t-u_{t-1})\|^2\big] \\
  & =  \mathbb{E}\big[ \eta^{-1}_{t-1} \|\nabla J(\theta_{t-1}) - u_{t-1} + \nabla J(\theta_t)-\nabla J(\theta_{t-1}) + \beta_tu_{t-1} - \beta_t w(\tau_t|\theta_{t},\theta_t(\alpha))g(\tau_t|\theta_t)
  - (1-\beta_t) \Delta_t\big] \nonumber \\
  & = \mathbb{E}\big[ \eta^{-1}_{t-1} \|(1-\beta_t)(\nabla J(\theta_{t-1}) - u_{t-1}) + \beta_t(\nabla J(\theta_t)- w(\tau_t|\theta_{t},\theta_t(\alpha))g(\tau_t|\theta_t))\nonumber \\
  & \quad - (1-\beta_t)\big( \Delta_t -(\nabla J(\theta_t)-\nabla J(\theta_{t-1}))\big)\|^2\big] \nonumber \\
  & = \eta^{-1}_{t-1}(1-\beta_t)^2\mathbb{E} \|\nabla J(\theta_{t-1}) - u_{t-1}\|^2 + \eta^{-1}_{t-1}\mathbb{E}\big[\|\beta_t(\nabla J(\theta_t)- w(\tau_t|\theta_{t},\theta_t(\alpha))g(\tau_t|\theta_t))\nonumber \\
  & \quad - (1-\beta_t)\big( \Delta_t-(\nabla J(\theta_t)-\nabla J(\theta_{t-1}))\big)\|^2\big] \nonumber \\
  & \leq \eta^{-1}_{t-1}(1-\beta_t)^2\mathbb{E} \|\nabla J(\theta_{t-1}) - u_{t-1}\|^2 + 2\beta^2_t\eta^{-1}_{t-1}\mathbb{E}\|\nabla J(\theta_t)- w(\tau_t|\theta_{t},\theta_t(\alpha))g(\tau_t|\theta_t)\|^2\nonumber \\
  & \quad +2(1-\beta_t)^2 \mathbb{E}\| \Delta_t -(\nabla J(\theta_t)-\nabla J(\theta_{t-1}))\|^2 \nonumber \\
  & \leq \eta^{-1}_{t-1}(1-\beta_t)^2\mathbb{E} \|e_{t-1}\|^2 + 2\beta^2_t\eta^{-1}_{t-1}\underbrace{\mathbb{E}\|w(\tau_t|\theta_{t},\theta_t(\alpha))g(\tau_t|\theta_t)\|^2}_{=T_3}
  + 2(1-\beta_t)^2\eta^{-1}_{t-1} \mathbb{E}\| \Delta_t\|^2, \nonumber
 \end{align}
 where the forth equality holds by $\mathbb{E}_{\tau_t\sim p(\tau|\theta_t(\alpha))}[w(\tau_t|\theta_{t},\theta_t(\alpha))g(\tau_t|\theta_t)]=\nabla J(\theta_t)$ and $\mathbb{E}_{\tau_t\sim p(\tau|\theta_t(\alpha))}[\Delta_t]=\nabla J(\theta_t)-\nabla J(\theta_{t-1})$; the first inequality follows by Young's inequality; and the last inequality holds by $\mathbb{E}\|\zeta-\mathbb{E}[\zeta]\|^2=\mathbb{E}\|\zeta\|^2-(\mathbb{E}[\zeta])^2 \leq \mathbb{E}\|\zeta\|^2$.

 Next, we give an upper bound of the term $T_3$ as follows:
 \begin{align} \label{eq:B2}
  T_3 & = \mathbb{E}\|w(\tau_t|\theta_{t},\theta_t(\alpha))g(\tau_t|\theta_t)\|^2 \nonumber \\
  & = \mathbb{E}\|w(\tau_t|\theta_{t},\theta_t(\alpha))g(\tau_t|\theta_t) -g(\tau_t|\theta_t) + g(\tau_t|\theta_t)\|^2\nonumber \\
  & \leq 2\mathbb{E}\|\big(w(\tau_t|\theta_{t},\theta_t(\alpha))-1\big)g(\tau_t|\theta_t)\|^2 + 2\mathbb{E}\|g(\tau_t|\theta_t)\|^2 \nonumber \\
  & \leq 2\mathbb{E}\|w(\tau_t|\theta_{t},\theta_t(\alpha))-1\|^2 \mathbb{E}\|g(\tau_t|\theta_t)\|^2 + 2\mathbb{E}\|g(\tau_t|\theta_t)\|^2 \nonumber \\
  & \leq  2(W+1)G_t^2,
 \end{align}
 where the last inequality holds by Proposition \ref{pro:1} and Assumption 3.

 Finally, combining the inequalities \eqref{eq:B1} with \eqref{eq:B2}, we have
 \begin{align}
  \mathbb{E}\big[ \eta^{-1}_{t-1} \|e_t\|^2\big]& \leq \eta^{-1}_{t-1}(1-\beta_t)^2\mathbb{E} \|e_{t-1}\|^2 + 4\beta^2_t\eta^{-1}_{t-1}(W+1)G_t^2
  + 2(1-\beta_t)^2\eta^{-1}_{t-1} \mathbb{E}\| \hat{\nabla}^2(\theta_t,\tau)v\|^2 \nonumber \\
  & \leq \eta^{-1}_{t-1}(1-\beta_t)^2\mathbb{E} \|e_{t-1}\|^2 + 4\beta^2_t\eta^{-1}_{t-1}(W+1)G_t^2 + 2(1-\beta_t)^2\eta^{-1}_{t-1} L^2 \mathbb{E}\| \theta_t - \theta_{t-1}\|^2 \nonumber \\
  & = \eta^{-1}_{t-1}(1-\beta_t)^2\mathbb{E} \|e_{t-1}\|^2 + 4\beta^2_t\eta^{-1}_{t-1}(W+1)G_t^2 + 2(1-\beta_t)^2L^2\eta_{t-1}\|e_{t-1} + \nabla J(\theta_{t-1})\|^2 \nonumber \\
  & \leq 4(W+1)\beta^2_t\eta^{-1}_{t-1}G_t^2 + \eta^{-1}_{t-1}(1-\beta_t)^2\big(1+4\eta^2_{t-1}L^2\big)\mathbb{E} \|e_{t-1}\|^2 + 4(1-\beta_t)^2L^2\eta_{t-1}\|\nabla J(\theta_{t-1})\|^2,\nonumber
 \end{align}
 where the second inequality holds by the Proposition 2.

\end{proof}

\begin{theorem} \label{th:C1}
Assume that the sequence $\{\theta_t\}_{t=1}^T$ be generated from Algorithm \ref{alg:2}, and let $k=O(\frac{G^{2/3}}{L})$ $c=\frac{G^2}{3k^3L}+52L^2$,
$m = \max\{2G^2,(2Lk)^3,(\frac{ck}{2L})^3\}$ and $\eta_0 = \frac{k}{m^{1/3}}$, we have
\begin{align}
  \mathbb{E}\|\nabla J(\theta_\zeta)\|=\frac{1}{T}\sum_{t=1}^T\mathbb{E}\|\nabla J(\theta_t)\| \leq \frac{\sqrt{2\Lambda}m^{1/6} + 2\Lambda^{3/4}}{\sqrt{T}} + \frac{2\sqrt{\Lambda}\sigma^{1/3}}{T^{1/3}}, \nonumber
\end{align}
 where $\Lambda=\frac{1}{k}\big(16(J^* - J(\theta_1)) + \frac{m^{1/3}}{4L^2k}\sigma^2 + \frac{(W+1)c^2k^{3}}{2L^2}\ln(T+2)\big)$ with $J^*=\sup_{\theta}J(\theta) <+\infty$.
\end{theorem}

\begin{proof}
This proof mainly follows the proof of the above Theorem \ref{th:A1}.
Due to $m\geq (2Lk)^3$, we have $\eta_t\leq \frac{k}{m^{1/3}}\leq \frac{1}{2L}$. Since $\eta_t\leq \frac{1}{2L}$ and
$m\geq (\frac{ck}{2L})^3$, we have $\beta_{t+1}=c\eta^2_t\leq \frac{c\eta_t}{2L}\leq \frac{ck}{2Lm^{1/3}}\leq 1$.
By Lemma \ref{lem:B1}, we have
 \begin{align} \label{eq:C2}
  \mathbb{E}[\eta^{-1}_{t}\|e_{t+1}\|^2-\eta^{-1}_{t-1}\|e_t\|^2] & \leq \mathbb{E}\big[  4(W+1)\beta^2_{t+1}\eta^{-1}_{t}G_{t+1}^2 + \eta^{-1}_{t}(1-\beta_{t+1})^2\big(1+4\eta^2_{t}L^2\big) \|e_{t}\|^2 \nonumber \\
  & \quad + 4(1-\beta_{t+1})^2L^2\eta_{t}\|\nabla J(\theta_{t})\|^2 -\eta^{-1}_{t-1}\|e_t\|^2 \big] \nonumber \\
  & \leq  \mathbb{E}\big[ 4(W+1)c^2\eta^{3}_{t}G_{t+1}^2 + \underbrace{ \big( \eta^{-1}_{t}(1-\beta_{t+1})(1+4\eta^2_{t}L^2) - \eta^{-1}_{t-1} \big) \|e_{t}\|^2 }_{=T_4}
  + 4L^2\eta_{t}\|\nabla J(\theta_{t})\|^2 \big],
 \end{align}
where the last inequality holds by $0< \beta_{t+1} \leq 1$.
Since the function $x^{1/3}$ is cancave, we have $(x+y)^{1/3}\leq x^{1/3} + yx^{-2/3}/3$.
Then we have
 \begin{align}
  \eta^{-1}_{t} - \eta^{-1}_{t-1}  &= \frac{1}{k}\bigg( \big(m+\sum_{i=1}^tG^2_i\big)^{1/3} - \big(m+\sum_{i=1}^{t-1}G^2_i\big)^{1/3} \bigg) \leq \frac{G^2_t}{3k(m+\sum_{i=1}^{t-1}G^2_i)^{2/3}} \nonumber \\
  & \leq \frac{G^2_t}{3k(m - G^2+\sum_{i=1}^{t}G^2_i)^{2/3}} \leq \frac{G^2_t}{3k(m/2 +\sum_{i=1}^{t}G^2_i)^{2/3}} \leq \frac{2^{2/3}G^2_t}{3k(m +\sum_{i=1}^{t}G^2_i)^{2/3}} \nonumber \\
  & \leq \frac{2^{2/3}G^2}{3k^3}\eta_t^2 \leq \frac{2^{2/3}G^2}{6k^3L}\eta_t\leq \frac{G^2}{3k^3L}\eta_t,
 \end{align}
where the third inequality holds by $m\geq 2G^2$, and the sixth inequality holds by $0<\eta\leq \frac{1}{2L}$.

Next, considering the upper bound of the term $T_4$, we have
 \begin{align} \label{eq:C3}
  T_4 &= \big( \eta^{-1}_{t}(1-\beta_{t+1})(1+4\eta^2_{t}L^2) -\eta^{-1}_{t-1} \big) \|e_{t}\|^2 \nonumber \\
  & = \big( \eta^{-1}_{t} - \eta^{-1}_{t-1} + 4L^2\eta_t - \beta_{t+1}\eta^{-1}_{t} - 4\eta_t\beta_{t+1}L^2 \big) \|e_{t}\|^2 \nonumber \\
  & \leq \big( \eta^{-1}_{t} - \eta^{-1}_{t-1} + 4L^2\eta_t - \beta_{t+1}\eta^{-1}_{t} \big) \|e_{t}\|^2 \nonumber \\
  & \leq  \big( \frac{G^2}{3k^3L}\eta_t + 4L^2\eta_t - c\eta_{t} \big) \|e_{t}\|^2= -48L^2\eta_t\|e_{t}\|^2,
 \end{align}
 where the last equality holds by $c=\frac{G^2}{3k^3L}+52L^2$.
 Combining the inequalities \eqref{eq:C2} with \eqref{eq:C3}, we have
 \begin{align} \label{eq:C4}
  \mathbb{E}[\eta^{-1}_{t}\|e_{t+1}\|^2-\eta^{-1}_{t-1}\|e_t\|^2] \leq  \mathbb{E}\big[ 4(W+1)c^2\eta^{3}_{t}G_{t+1}^2 -48L^2\eta_t\|e_{t}\|^2
  + 4L^2\eta_{t}\|\nabla J(\theta_{t})\|^2 \big].
 \end{align}

We define a \emph{Lyapunov} function $\Psi_t = J(\theta_t) - \frac{1}{64L^2\eta_{t-1}}\|e_t\|^2$ for any $t\geq 1$.
Then we have
 \begin{align} \label{eq:C5}
  \mathbb{E}[\Psi_{t+1} - \Psi_t] & = \mathbb{E}\big[ J(\theta_{t+1}) - J(\theta_t) - \frac{1}{64L^2\eta_{t}} \|e_{t+1}\|^2 + \frac{1}{64L^2\eta_{t-1}}\|e_t\|^2 \big] \nonumber \\
  & \geq \mathbb{E}\big[- \frac{3\eta_t}{4}\|e_t\|^2 + \frac{\eta_t}{8}\|\nabla J(\theta_t)\|^2 - \frac{1}{64L^2}(\eta^{-1}_{t}\|e_{t+1}\|^2 - \eta^{-1}_{t-1}\|e_t\|^2)\big] \nonumber \\
  & \geq -\frac{(W+1)c^2\eta^{3}_{t}G_{t+1}^2 }{32L^2}+ \frac{\eta_t}{16}\mathbb{E}\|\nabla J(\theta_t)\|^2,
 \end{align}
where the first inequality holds by the Lemma \ref{lem:A3}, and the second inequality follows by the above inequality \eqref{eq:C4}.
Summing the above inequality \eqref{eq:C5} over $t$ from $1$ to $T$, we obtain
\begin{align} \label{eq:C6}
 \sum_{t=1}^T\mathbb{E}[\eta_t\|\nabla J(\theta_t)\|^2] & \leq \mathbb{E}[16(\Psi_{T+1} - \Psi_1)] + \sum_{t=1}^T\frac{(W+1)c^2\eta^{3}_{t}G_{t+1}^2 }{2L^2} \nonumber \\
 & \leq \mathbb{E}[16(J^* - J(\theta_1))] + \frac{1}{4L^2\eta_0}\mathbb{E}\|e_1\|^2 + \frac{(W+1)c^2k^{3}}{2L^2}\sum_{t=1}^T\frac{G_{t+1}^2 }{m+\sum_{i=1}^tG_i^2} \nonumber \\
 & \leq \mathbb{E}[16(J^* - J(\theta_1))] + \frac{1}{4L^2\eta_0}\mathbb{E}\|e_1\|^2 + \frac{(W+1)c^2k^{3}}{2L^2}\sum_{t=1}^T\frac{G_{t+1}^2 }{G^2+\sum_{i=1}^{t+1}G_i^2}  \nonumber \\
 & \leq \mathbb{E}[16(J^* - J(\theta_1))] + \frac{1}{4L^2\eta_0}\mathbb{E}\|e_1\|^2 + \frac{(W+1)c^2k^{3}}{2L^2}\sum_{t=1}^T\big(\ln(G^2+ \sum_{i=1}^{t+1}G_i^2) - \ln(G^2+ \sum_{i=1}^{t}G_i^2)\big)  \nonumber \\
 & \leq \mathbb{E}[16(J^* - J(\theta_1))] + \frac{1}{4L^2\eta_0}\mathbb{E}\|e_1\|^2 + \frac{(W+1)c^2k^{3}}{2L^2}(\ln(G^2+ \sum_{i=1}^{T+1}G_i^2) - \ln(G^2)) \nonumber \\
 & \leq \mathbb{E}[16(J^* - J(\theta_1))] + \frac{m^{1/3}}{4L^2k}\sigma^2 + \frac{(W+1)c^2k^{3}}{2L^2}\ln(1+ \sum_{i=1}^{T+1}\frac{G_i^2}{G^2}) \nonumber \\
 & \leq \mathbb{E}[16(J^* - J(\theta_1))] + \frac{m^{1/3}}{4L^2k}\sigma^2 + \frac{(W+1)c^2k^{3}}{2L^2}\ln(T+2),
\end{align}
where $J^*=\sup_{\theta}J(\theta) <+\infty$, and the fourth inequality holds by the concavity of the function $\ln(x)$, and the sixth inequality holds by the definition of $e_1$ and $\eta_0$.

By Cauchy-Schwarz inequality, we have $\mathbb{E}[XY]^2 \leq\mathbb{E}[X^2]\mathbb{E}[Y^2]$.
Let $X=\sqrt{\eta_T\sum_{t=1}^{T-1}\|\nabla J(\theta_t)\|^2}$ and $Y=\sqrt{1/\eta_T}$,
we have
\begin{align}
 \mathbb{E}\big[\sum_{t=1}^T\|\nabla J(\theta_t)\|^2\big] \leq \mathbb{E}[1/\eta_T]\mathbb{E}\big[ \eta_T\sum_{t=1}^T\|\nabla J(\theta_t)\|^2\big].
\end{align}
Since $\eta_t$ is decreasing, we have
\begin{align} \label{eq:C7}
  \mathbb{E}\big[\sum_{t=1}^T\|\nabla J(\theta_t)\|^2\big] \leq \mathbb{E}[1/\eta_T]\mathbb{E}\big[\sum_{t=1}^T\eta_T\|\nabla J(\theta_t)\|^2\big] \leq \mathbb{E}[1/\eta_T]\mathbb{E}\big[\sum_{t=1}^T\eta_t\|\nabla J(\theta_t)\|^2\big].
\end{align}
Combining the inequalities \eqref{eq:C6} and \eqref{eq:C7}, we obtain
\begin{align} \label{eq:C8}
 \mathbb{E}\big[\sum_{t=1}^T\|\nabla J(\theta_t)\|^2\big] &\leq \mathbb{E}\big[ \frac{16(J^* - J(\theta_1)) + \frac{m^{1/3}}{4L^2k}\sigma^2 + \frac{(W+1)c^2k^{3}}{2L^2}\ln(T+2)}{\eta_T} \big] \nonumber \\
 & = \mathbb{E}\big[ \Lambda\big(m+\sum_{t=1}^TG^2_t\big)^{1/3} \big]
\end{align}
where $\Lambda=\frac{1}{k}\big(16(J^* - J(\theta_1)) + \frac{m^{1/3}}{4L^2k}\sigma^2 + \frac{(W+1)c^2k^{3}}{2L^2}\ln(T+2)\big)$.

By Assumption 2, we have $G^2_t = \|g(\tau|\theta_t)-\nabla J(\theta_t)+\nabla J(\theta_t)\|^2 \leq 2\|g(\tau|\theta_t)-\nabla J(\theta_t)\|^2 + 2\|\nabla J(\theta_t)\|^2 \leq 2\sigma^2+2\|\nabla J(\theta_t)\|^2$.
Then using the inequality $(a+b)^{1/3}\leq a^{1/3} + b^{1/3}$ for all $a,b>0$ to the inequality \eqref{eq:C8}, we obtain
\begin{align}
 \bigg(\mathbb{E}\sqrt{\sum_{t=1}^T\|\nabla J(\theta_t)\|^2}\bigg)^2 & \leq \mathbb{E}\bigg[\sqrt{\sum_{t=1}^T\|\nabla J(\theta_t)\|^2}\bigg]^2 = \mathbb{E}\big[\sum_{t=1}^T\|\nabla J(\theta_t)\|^2\big] \nonumber \\
 & \leq \mathbb{E}\bigg[\Lambda(m+2T\sigma^2)^{1/3} + 2^{1/3}\Lambda\big(\sum_{t=1}^T\|\nabla J(\theta_t)\|^2\big)^{1/3}\bigg] \nonumber \\
 & =\Lambda(m+2T\sigma^2)^{1/3} + 2^{1/3}\Lambda\mathbb{E}\bigg[\sqrt{\sum_{t=1}^T\|\nabla J(\theta_t)\|^2}\bigg]^{2/3} \nonumber \\
 & \leq \Lambda(m+2T\sigma^2)^{1/3} + 2^{1/3}\Lambda\bigg[\mathbb{E}\sqrt{\sum_{t=1}^T\|\nabla J(\theta_t)\|^2}\bigg]^{2/3},
\end{align}
where the first inequality holds by the convexity of the function $x^2$, and the last inequality holds by the concavity of the function $x^{2/3}$.
For simplicity, let $Z=\sqrt{\sum_{t=1}^T\|\nabla J(\theta_t)\|^2}$, we have
\begin{align} \label{eq:C9}
 \big(\mathbb{E}[Z]\big)^2 \leq \Lambda(m+2T\sigma^2)^{1/3} + 2^{1/3}\Lambda\big(\mathbb{E}[Z]\big)^{2/3}.
\end{align}
The inequality \eqref{eq:C9} implies that $ \big(\mathbb{E}[Z]\big)^2 \leq 2\Lambda(m+2T\sigma^2)^{1/3}$ or $ \big(\mathbb{E}[Z]\big)^2 \leq 2\cdot2^{1/3}\Lambda\big(\mathbb{E}[Z]\big)^{2/3}$.
Thus, we have
\begin{align}
 \mathbb{E}[Z] \leq \sqrt{2\Lambda}(m+2T\sigma^2)^{1/6} + 2\Lambda^{3/4}.
\end{align}
By Cauchy-Schwarz inequality, then we have
\begin{align}
  \frac{1}{T}\sum_{t=1}^T\mathbb{E}\|\nabla J(\theta_t)\| &= \mathbb{E} \big[\frac{1}{T}\sum_{t=1}^T\|\nabla J(\theta_t)\|\big] \leq \mathbb{E}\bigg[\sqrt{\frac{1}{T}}
  \sqrt{\sum_{t=1}^T\|\nabla J(\theta_t)\|^2}\bigg] \nonumber \\
  & \leq \frac{\sqrt{2\Lambda}(m+2T\sigma^2)^{1/6} + 2\Lambda^{3/4}}{\sqrt{T}} \nonumber \\
  & \leq \frac{\sqrt{2\Lambda}m^{1/6} + 2\Lambda^{3/4}}{\sqrt{T}} + \frac{2\sqrt{\Lambda}\sigma^{1/3}}{T^{1/3}},
\end{align}
where the last inequality follows by the inequality $(a+b)^{1/6}\leq a^{1/6} + b^{1/6}$ for all $a,b>0$.
\end{proof}

\subsection{ Convergence Analysis of IS-MBPG* Algorithm}
\label{Appendix:A3}
In this subsection, we detailedly provide the convergence properties of our IS-MBPG* algorithm.

\begin{lemma} \label{lem:D1}
 Assume that the stochastic policy gradient $u_t$ be generated from Algorithm \ref{alg:3}, and let $e_t = u_t-\nabla J(\theta_t)$,  we have
 \begin{align}
\mathbb{E}\big[ \eta^{-1}_{t-1}\|e_t\|^2 \big]
\leq 2\beta^2_t\eta^{-1}_{t-1}\sigma^2 + \eta^{-1}_{t-1}(1-\beta_t)^2\big(1+8\eta^2_{t-1}B^2\big)\mathbb{E} \|e_{t-1}\|^2 + 8(1-\beta_t)^2B^2\eta_{t-1}\|\nabla J(\theta_{t-1})\|^2, \nonumber
 \end{align}
 where $B^2 = L^2 + 2G^2C^2_w$ with $C_w = \sqrt{H(2HM_g^2+M_h)(W+1)}$.
\end{lemma}
\begin{proof}
 The proof is the similar to that of Lemma \ref{lem:A4}. The only difference is that instead of using $2\beta^2_t\eta^{-1}_{t-1}\mathbb{E}\|\nabla J(\theta_t) - g(\tau_t|\theta_t)\|^2 \leq 2\beta^2_t\eta^{-1}_{t-1}\sigma^2$
 instead of $2\beta^2_t\eta^{-1}_{t-1}\mathbb{E}\|\nabla J(\theta_t) -g(\tau_t|\theta_t)\|^2 \leq 2\beta^2_t\eta^{-1}_{t-1}\mathbb{E}\|g(\tau_t|\theta_t)\|^2 = 2\beta^2_t\eta^{-1}_{t-1}G_t^2$.
\end{proof}

\begin{theorem} \label{th:D1}
Assume that the sequence $\{\theta_t\}_{t=1}^T$ be generated from Algorithm \ref{alg:3}, and let $B^2 = L^2 + 2G^2C^2_w$, $k> 0$ $c=\frac{1}{3k^3L}+104B^2$,
$m = \max\{2,(2Lk)^3,(\frac{ck}{2L})^3\}$ and $\eta_0 = \frac{k}{m^{1/3}}$, we have
\begin{align}
  \mathbb{E}\|\nabla J(\theta_\zeta)\|=\frac{1}{T}\sum_{t=1}^T\mathbb{E}\|\nabla J(\theta_t)\| \leq \frac{\sqrt{\Gamma}m^{1/6}}{\sqrt{T}} + \frac{\sqrt{\Gamma}}{T^{1/3}}, \nonumber
\end{align}
 where $\Gamma=\frac{1}{k}\big(16(J^* - J(\theta_1)) + \frac{m^{1/3}}{8B^2k}\sigma^2 + \frac{c^2k^{3}\sigma^2}{4B^2}\ln(T+2)\big)$ with $J^*=\sup_{\theta}J(\theta) <+\infty$.
\end{theorem}
\begin{proof}
This proof mainly follows the proof of the above Theorem \ref{th:A1}.
Due to $m\geq (2Lk)^3$, we have $\eta_t\leq \frac{k}{m^{1/3}}\leq \frac{1}{2L}$. Since $\eta_t\leq \frac{1}{2L}$ and
$m\geq (\frac{ck}{2L})^3$, we have $\beta_{t+1}=c\eta^2_t\leq \frac{c\eta_t}{2L}\leq \frac{ck}{2Lm^{1/3}}\leq 1$.
By Lemma \ref{lem:D1}, we have
 \begin{align} \label{eq:D1}
  \mathbb{E}[\eta^{-1}_{t}\|e_{t+1}\|^2-\eta^{-1}_{t-1}\|e_t\|^2] & \leq \mathbb{E}\big[ 2\beta^2_{t+1}\eta^{-1}_{t}\sigma^2 + \eta^{-1}_{t}(1-\beta_{t+1})^2\big(1+8\eta^2_{t}B^2\big) \|e_{t}\|^2 \nonumber \\
  & \quad + 8(1-\beta_{t+1})^2B^2\eta_{t}\|\nabla J(\theta_{t})\|^2 -\eta^{-1}_{t-1}\|e_t\|^2 \big] \nonumber \\
  & \leq  \mathbb{E}\big[ 2c^2\eta^{3}_{t}\sigma^2 + \underbrace{ \big( \eta^{-1}_{t}(1-\beta_{t+1})(1+8\eta^2_{t}B^2) -\eta^{-1}_{t-1} \big) \|e_{t}\|^2 }_{=T_5}
  + 8B^2\eta_{t}\|\nabla J(\theta_{t})\|^2 \big],
 \end{align}
where the last inequality holds by $0< \beta_{t+1} \leq 1$.
Since the function $x^{1/3}$ is cancave, we have $(x+y)^{1/3}\leq x^{1/3} + yx^{-2/3}/3$.
Then we have
 \begin{align}
  \eta^{-1}_{t} - \eta^{-1}_{t-1}  &= \frac{1}{k}\bigg( \big(m+t\big)^{1/3} - \big(m+t-1\big)^{1/3} \bigg) \leq \frac{1}{3k(m+t-1)^{2/3}} \nonumber \\
  & \leq \frac{1}{3k(m/2 +t)^{2/3}} \leq \frac{2^{2/3}}{3k(m +t)^{2/3}} \nonumber \\
  & \leq \frac{2^{2/3}}{3k^3}\eta_t^2 \leq \frac{2^{2/3}}{6k^3L}\eta_t\leq \frac{1}{3k^3L}\eta_t,
 \end{align}
where the second inequality holds by $m\geq 2$, and the fifth inequality holds by $0<\eta\leq \frac{1}{2L}$.

Next, considering the upper bound of the term $T_5$, we have
 \begin{align} \label{eq:D2}
  T_5 &= \big( \eta^{-1}_{t}(1-\beta_{t+1})(1+8\eta^2_{t}B^2) -\eta^{-1}_{t-1} \big) \|e_{t}\|^2 \nonumber \\
  & = \big( \eta^{-1}_{t} - \eta^{-1}_{t-1} + 8B^2\eta_t - \beta_{t+1}\eta^{-1}_{t} - 8\eta_t\beta_{t+1}B^2 \big) \|e_{t}\|^2 \nonumber \\
  & \leq \big( \eta^{-1}_{t} - \eta^{-1}_{t-1} + 8B^2\eta_t - \beta_{t+1}\eta^{-1}_{t} \big) \|e_{t}\|^2 \nonumber \\
  & \leq  \big( \frac{1}{3k^3L}\eta_t + 8B^2\eta_t - c\eta_{t} \big) \|e_{t}\|^2= -96B^2\eta_t\|e_{t}\|^2,
 \end{align}
 where the last equality holds by $c=\frac{1}{3k^3L}+104B^2$. Combining the inequalities \eqref{eq:D1} with \eqref{eq:D2}, we have
 \begin{align} \label{eq:D3}
  \mathbb{E}[\eta^{-1}_{t}\|e_{t+1}\|^2-\eta^{-1}_{t-1}\|e_t\|^2] \leq  \mathbb{E}\big[ 2c^2\eta^{3}_{t}\sigma^2 -96B^2\eta_t\|e_{t}\|^2
  + 8B^2\eta_{t}\|\nabla J(\theta_{t})\|^2 \big].
 \end{align}

We define a \emph{Lyapunov} function $\Phi_t = J(\theta_t) - \frac{1}{128B^2\eta_{t-1}}\|e_t\|^2$ for any $t\geq 1$.
Then we have
 \begin{align} \label{eq:D4}
  \mathbb{E}[\Phi_{t+1} - \Phi_t] & = \mathbb{E}\big[ J(\theta_{t+1}) - J(\theta_t) - \frac{1}{128B^2\eta_{t}} \|e_{t+1}\|^2 + \frac{1}{128B^2\eta_{t-1}}\|e_t\|^2 \big] \nonumber \\
  & \geq \mathbb{E}\big[- \frac{3\eta_t}{4}\|e_t\|^2 + \frac{\eta_t}{8}\|\nabla J(\theta_t)\|^2 - \frac{1}{128B^2}(\eta^{-1}_{t}\|e_{t+1}\|^2 - \eta^{-1}_{t-1}\|e_t\|^2)\big] \nonumber \\
  & \geq -\frac{c^2\eta^{3}_{t}\sigma^2 }{64B^2}+ \frac{\eta_t}{16}\mathbb{E}\|\nabla J(\theta_t)\|^2,
 \end{align}
where the first inequality holds by the Lemma \ref{lem:A3}, and the second inequality follows by the above inequality \eqref{eq:D3}.
Summing the above inequality \eqref{eq:D4} over $t$ from $1$ to $T$, we obtain
\begin{align} \label{eq:D5}
 \sum_{t=1}^T\mathbb{E}[\eta_t\|\nabla J(\theta_t)\|^2] & \leq \mathbb{E}[16(\Phi_{T+1} - \Phi_1)] + \sum_{t=1}^T\frac{c^2\eta^{3}_{t}\sigma^2 }{4B^2} \nonumber \\
 & \leq \mathbb{E}[16(J^* - J(\theta_1))] + \frac{1}{8B^2\eta_0}\mathbb{E}\|e_1\|^2 + \frac{c^2k^{3}\sigma^2}{4B^2}\sum_{t=1}^T\frac{1}{m+t} \nonumber \\
 & \leq \mathbb{E}[16(J^* - J(\theta_1))] + \frac{1}{8B^2\eta_0}\mathbb{E}\|e_1\|^2 + \frac{c^2k^{3}\sigma^2}{4B^2}\sum_{t=1}^T\frac{1}{t+2}  \nonumber \\
 & \leq \mathbb{E}[16(J^* - J(\theta_1))] + \frac{1}{8B^2\eta_0}\sigma^2 + \frac{c^2k^{3}\sigma^2}{4B^2}\sum_{t=1}^T\frac{1}{t+2}  \nonumber \\
 & \leq \mathbb{E}[16(J^* - J(\theta_1))] + \frac{m^{1/3}}{8B^2k}\sigma^2 + \frac{c^2k^{3}\sigma^2}{4B^2}\ln(T+2),
\end{align}
where $J^*=\sup_{\theta}J(\theta) <+\infty$, and the third inequality is due to $m\geq 2$,
and the last inequality holds by $\sum_{t=1}^T\frac{1}{t+2} \leq \int^T_1\frac{1}{t+2}dt \leq \ln(T+2)$.

Since $\eta_t=\frac{k}{(m+t)^{1/3}}$ is decreasing, we have
\begin{align} \label{eq:D6}
  \sum_{t=1}^T\mathbb{E}\|\nabla J(\theta_t)\|^2 &\leq 1/\eta_T\sum_{t=1}^T\mathbb{E}\big[ \eta_t\|\nabla J(\theta_t)\|^2\big] \nonumber \\
  &\leq  \frac{16(J^* - J(\theta_1)) + \frac{m^{1/3}}{8B^2k}\sigma^2 + \frac{c^2k^{3}}{4B^2}\ln(T+2)}{\eta_T}\nonumber \\
  & = \Omega\big(m+T\big)^{1/3}
\end{align}
where $\Gamma=\frac{1}{k}\big(16(J^* - J(\theta_1)) + \frac{m^{1/3}}{8B^2k}\sigma^2 + \frac{c^2k^{3}\sigma^2}{4B^2}\ln(T+2)\big)$.

According to Jensen's inequality, we have
\begin{align} \label{eq:D6}
  \frac{1}{T}\sum_{t=1}^T \mathbb{E}\|\nabla J(\theta_t)\| & \leq \big( \frac{1}{T}\sum_{t=1}^T \mathbb{E}\|\nabla J(\theta_t)\|^2\big)^{1/2}
  \leq\frac{\sqrt{\Gamma}\big(m+T\big)^{1/6}}{\sqrt{T}}
  \leq \frac{\sqrt{\Gamma}m^{1/6}}{\sqrt{T}} + \frac{\sqrt{\Gamma}}{T^{1/3}},
\end{align}
where the last inequality follows by the inequality $(a+b)^{1/6}\leq a^{1/6} + b^{1/6}$ for all $a,b>0$.
\end{proof}

\subsection{ Convergence Analysis of Vanilla Policy Gradient Algorithm }
\label{Appendix:A4}
In the subsection, we provide a detailed theoretical analysis of
vanilla policy gradient method such as REINFORCE \cite{williams1992simple}.
In fact, the sample complexity $O(\epsilon^{-4})$ does not directly come from \cite{williams1992simple}, but follows theoretical results of SGD \cite{ghadimi2013stochastic}. To make convincing, we provide a detailed theoretical analysis about this result in the following.
We first give the algorithmic framework of vanilla policy gradient method such as REINFORCE \cite{williams1992simple} in Algorithm \ref{alg:4}.

\begin{algorithm}[tb]
\caption{ Vanilla Policy Gradient Algorithm}
\label{alg:4}
\begin{algorithmic}[1] 
\STATE {\bfseries Input:}  Total iteration $T$, mini-batch size $N$, learning rate $\eta$ and initial input $\theta_1$; \\
\FOR{$t = 1, 2, \ldots, T$}
\STATE Sample a mini-batch trajectories $\{\tau_i\}_{i=1}^N$ from $p(\tau |\theta_t)$; \\
\STATE Compute $u_t =  \frac{1}{N}\sum_{i=1}^N g(\tau_i|\theta_t)$; \\
\STATE Update $\theta_{t+1} = \theta_t + \eta u_t$;
\ENDFOR
\STATE {\bfseries Output:}  $\theta_{\zeta}$ chosen uniformly random from $\{\theta_t\}_{t=1}^{T}$.
\end{algorithmic}
\end{algorithm}

\begin{theorem} \label{th:E7}
Under Assumption 1, let $0< \eta \leq \frac{1}{2L}$,
we have
\begin{align}
 \frac{1}{T}\sum_{t=1}^T \mathbb{E}\|\nabla J(\theta_t)\| \leq \frac{\sqrt{\eta\big(J^*-J(\theta_{1})\big)}}{2\sqrt{2T}} + \frac{\sqrt{6}\sigma}{\sqrt{N}},
\end{align}
where $J^*=\sup_{\theta}J(\theta) <+\infty$.
\end{theorem}

\begin{proof}
By using $J(\theta)$ is $L$-smooth, we have
\begin{align} \label{eq:E1}
 \mathbb{E} [J(\theta_{t+1})] & \geq  \mathbb{E} [J(\theta_{t}) + \nabla J(\theta_{t})^T(\theta_{t+1}-\theta_{t}) - \frac{L}{2}\|\theta_{t+1}-\theta_{t}\|^2] \nonumber \\
 & = \mathbb{E} [J(\theta_{t}) + \eta\nabla J(\theta_{t})^Tu_t - \frac{L\eta^2}{2}\|u_{t}\|^2] \nonumber \\
 & = \mathbb{E} [J(\theta_{t}) + \eta(\nabla J(\theta_{t})-u_t)^Tu_t + \eta\|u_t\|^2 - \frac{L\eta^2}{2}\|u_{t}\|^2] \nonumber \\
 & \geq  \mathbb{E} [J(\theta_{t}) - \frac{\eta}{2}\|\nabla J(\theta_{t})-u_t\|^2 + \frac{\eta}{2}(1-L\eta)\|u_t\|^2] \nonumber \\
 & \geq \mathbb{E} [J(\theta_{t}) - \frac{\eta}{2}\|\nabla J(\theta_{t})-u_t\|^2 + \frac{\eta}{4}\|u_t\|^2] \nonumber \\
 & \geq \mathbb{E} [J(\theta_{t}) - \frac{3\eta}{4}\|\nabla J(\theta_{t})-u_t\|^2 + \frac{\eta}{8}\|\nabla J(\theta_t)\|^2 ] \nonumber \\
 & \geq \mathbb{E} [J(\theta_{t}) - \frac{3\eta\sigma^2}{4N} + \frac{\eta}{8}\|\nabla J(\theta_t)\|^2 ],
\end{align}
where the second inequality holds by Young's inequality, and the third inequality holds by $0< \eta \leq \frac{1}{2L}$, and
the fourth inequality follows by $\|\nabla J(\theta_t)\|^2 \leq 2\|u_t\|^2 + 2\|\nabla J(\theta_t)-u_t\|^2$,
and the final inequality holds by $\mathbb{E}\|\nabla J(\theta_{t})-u_t\|^2=\mathbb{E}\|\nabla J(\theta_{t})-\frac{1}{N}\sum_{i=1}^N g(\tau_i|\theta_t)\|^2\leq \frac{\sigma^2}{N}$ by Assumption 2.

Telescope the inequality \eqref{eq:E1} from $t=1$ to $T$, we have
\begin{align}
    \frac{1}{T}\sum_{t=1}^T \mathbb{E}\|\nabla J(\theta_t)\|^2 \leq \frac{\eta\big(J^*-J(\theta_{1})\big)}{8T} + \frac{6\sigma^2}{N}.
\end{align}
According to Jensen's inequality, then we have
\begin{align}
   \frac{1}{T}\sum_{t=1}^T \mathbb{E}\|\nabla J(\theta_t)\| \leq  \big( \frac{1}{T}\sum_{t=1}^T \mathbb{E}\|\nabla J(\theta_t)\|^2 \big)^{1/2} \leq \frac{\sqrt{\eta\big(J^*-J(\theta_{1})\big)}}{2\sqrt{2T}} + \frac{\sqrt{6}\sigma}{\sqrt{N}}.
\end{align}

\end{proof}

\begin{remark}
In Algorithm \ref{alg:4}, given $0< \eta \leq \frac{1}{2L}$, $N=O(\frac{1}{\epsilon^2})$ and $T=O(\frac{1}{\epsilon^2})$,
we have $\frac{1}{T}\sum_{t=1}^T \mathbb{E}\|\nabla J(\theta_t)\| \leq \epsilon$.
Thus, the vanilla policy gradient method such as REINFORCE has the sample complexity of $O(\epsilon^{-4})$ for finding an $\epsilon$-stationary point.
\end{remark}

\end{appendices}

\end{onecolumn}

\end{document}